\documentclass{aimsessay}

\usepackage{booktabs}
\usepackage{array}
\usepackage{caption}

\usepackage{algorithm}
\usepackage{algorithmicx}
\usepackage{algpseudocode}

\usepackage[below]{placeins}
\usepackage{sidecap}
\usepackage{graphicx}
\usepackage{moreverb}
\usepackage{hyperref}
\usepackage{fancyhdr} %For header image on title page
 
\usepackage[latin1]{inputenc}
\usepackage{wrapfig}
\usepackage[numbers]{natbib} 

\swapnumbers 

\theoremstyle{plain}
\newtheorem{thm}[subsection]{Theorem}
\theoremstyle{definition}

\newtheorem{conj}[subsection]{Conjecture}
\newtheorem{pro}[subsection]{Proposition}
\newtheorem{exa}[subsection]{Example}
\newtheorem{defn}[subsection]{Definition}
\newtheorem{rem}[subsection]{Remark}

\numberwithin{equation}{section}

\title{\vspace*{-4cm} Topological Foundations of
	Reinforcement Learning}
\author{
	Krame Kadurha David \\
	\texttt{david.krame@aims-cameroon.org} \\
	African Institute for Mathematical Sciences (AIMS) \\
	Cameroon \\
	\vspace{0.5cm} % Adds a small space between your info and the supervisor's name
	{\small Supervised by: Dr. Yaé U. Gaba} \\
	{\small AIMS-Research \& Innovation Center (AIMS RIC)} \\
	{\small Kigali, Rwanda}
}
\date{{\small 18 May 2024}\\
  {\scriptsize\it Submitted in Partial Fulfillment of a Structured Masters Degree at AIMS-Cameroon}
}
\begin{document}
\fancypagestyle{empty}{
  \fancyhf{}
  \renewcommand{\headrulewidth}{0pt}
  \fancyhead[C]{\includegraphics[width = \textwidth]{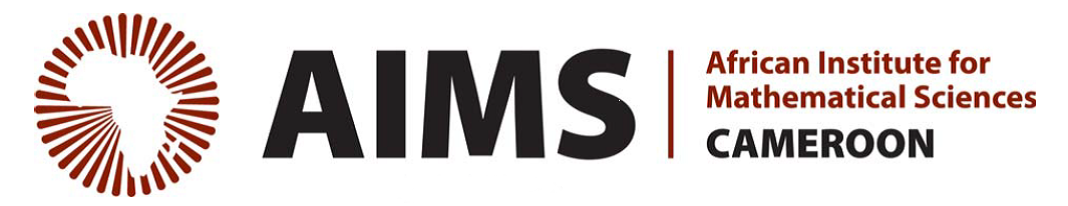}} % Logo
}

\pagestyle{empty}

\setlength{\headheight}{95pt} 

\maketitle

\pagenumbering{roman}
% Abstracts are usually written in English, with a version in your
% mother tongue underneath
\chapter*{Abstract} 
\addcontentsline{toc}{chapter}{Abstract}
% Don't change anything above this.
The goal of this work is to serve as a foundation for deep studies of the topology of state, action, and policy spaces in reinforcement learning. By studying these spaces from a mathematical perspective, we expect to gain more insight into how to build better algorithms to solve decision problems. Therefore, we focus on presenting the connection between the Banach fixed point theorem and the convergence of reinforcement learning algorithms, and we illustrate how the insights gained from this can practically help in designing more efficient algorithms. Before doing so, however, we first introduce relevant concepts such as metric spaces, normed spaces and Banach spaces for better understanding, before expressing the entire reinforcement learning problem in terms of Markov decision processes. This allows us to properly introduce the Banach contraction principle in a language suitable for reinforcement learning, and to write the Bellman equations in terms of operators on Banach spaces to show why reinforcement learning algorithms converge. Finally, we show how the insights gained from the mathematical study of convergence are helpful in reasoning about the best ways to make reinforcement learning algorithms more efficient.

\textbf{Keywords :} Markov Decision Process; Reinforcement Learning; Contraction mapping; Fixed Point; Banach Space; Bellman Operators; Q-Learning.

\vfill
\section*{Declaration}
I, the undersigned, hereby declare that the work contained in this essay is my original work, and that any work done by others or by myself previously has been acknowledged and referenced accordingly.

% Scan your signature into a small picture called 'signature.png' and insert it
% above your name and the date:
\includegraphics[height=4cm]{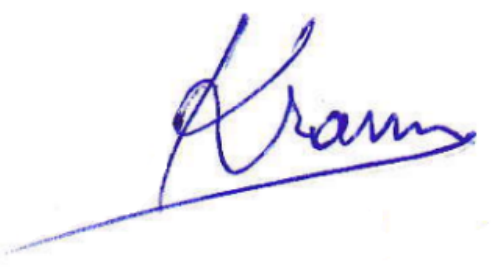} \hrule
% Your name must be in English Capitalisation with no comma, and the Family name comes last. 
% Do note the date below. It is called the "deadline".
Krame Kadurha David, 18 May 2024.

\tableofcontents

\listoffigures
\addcontentsline{toc}{chapter}{List of Figures}
%-----------------------------------------------------------------------------
\newpage

\pagenumbering{arabic}
\pagestyle{myheadings}
%-----------------------------------------------------------------------------

\chapter{Introduction}
There is not much work on the foundations of Reinforcement Learning (RL). Especially from a topological point of view. This work is, to the best of our knowledge, the first attempt to unify the basic concepts of Reinforcement Learning in a coherent and well-ordered way, specifying all the underlying mathematical concepts.

Some researchers have instead been interested in related questions, trying to explore RL concepts from a mathematical point of view, with the aim of improving the effectiveness of RL algorithms. For example:

In the paper \cite{ValuePolytope}, the authors studied the geometric and topological properties of the value functions in Markov decision processes with finite number of states and actions. Their characterization of the value function space as a general polytope provides an understanding of the relationships between policies and value functions. This perspective enhances the theoretical understanding of RL dynamics. The authors of \cite{DBLP:journals/corr/abs-1804-07193} also explored the concept of Lipschitz continuity in model-based RL, emphasizing the importance of learning Lipschitz continuous models for efficient value function estimation. This establishes some error bounds and demonstrates the benefits of controlling the Lipschitz constant in RL setting. The work done in \cite{DBLP:journals/corr/abs-1905-00475} extends model-free RL algorithms to continuous state-action spaces using a Q-learning-based approach, and that advancement opens avenues for applying RL techniques to a broader range of real-world applications with continuous state spaces. A very deep study is then done in \cite{MetricAndContinuityInRL}, introducing a unified formalism for defining state similarity metrics in RL. The authors claim to address the challenge of generalization in continuous-state systems by establishing hierarchies among metrics and examining their implications on RL algorithms.\newline
The aforementioned works can be difficult to understand or even less precise if the foundations are not well clarified. Therefore, the main goal of our work is to serve as a first stone in studying and organising the basic concepts, while presenting Reinforcement Learning in a simple and concise way, to help future researchers to quickly dive into the foundations of RL and to study more deeply, in particular the topology of state, action, and policy spaces. By studying these spaces from a mathematical perspective, we expect to gain more insight into how to build better algorithms to solve decision problems.\newline
At the end of this work, we even illustrate the latter assumption by showing that, with the mathematical investigation we have done, we are able to make very strong and general conjectures which merit to be taken seriously into account. This work is organized then as follows :
 
After this introduction, in chapter \ref{chap1}, we discuss metric spaces and their completeness, as well as the Banach Fixed Point Theorem. We end the chapter giving an application of the Banach Fixed Point Theorem to a concrete example.

In chapter \ref{chap2}, we formally present the Markov Decision Process framework and directly after that, we present the concept of optimality in Reinforcement Learning. At the end, we then discuss some methods which are at the core of how optimality is attained in Reinforcement Learning.

In chapter \ref{chap3}, we express reinforcement in term of Operators to have a simple way of proving why Reinforcement Learning Algorithms work, using the Banach Contraction Principle. In fact, this chapter discusses about Bellman Operators and shows how the optimality can be reached using operators' language.

Finally, before the final conclusion, in chapter \ref{chap4} we present some limitations of the classical Bellman operator and, using some insights gained from the mathematical investigation done before, we propose an alternative to the classical Bellman operator which shows good behaviour in terms of optimality and efficiency.
\chapter{Metric spaces and the Banach fixed point}\label{chap1}
% \section*{Partial Introduction}
% In this chapter, we discuss metric spaces and their completeness because it is necessary to think about it before presenting, in a second point of the chapter, notions about contraction mappings and the Banach Fixed Point Theorem. At the end now, we show an example of how the Banach Fixed Point Theorem can be used to solve a real problem requiring to prove existence and unicity of solutions.

%In this chapter, we discuss metric spaces and their completeness, as well as the Banach Fixed Point Theorem. We end the chapter giving an application of the Banach Fixed Point Theorem on a concrete example.

In this chapter, we present the fundamental concepts of metric spaces and explore their completeness, a crucial property in understanding the behavior of sequences within these spaces. We also study the Banach fixed point theorem, a result which establishes the existence and uniqueness of fixed points for mappings in complete metric spaces. We elucidate the significance of this theorem in providing a theoretical foundation for iterative algorithms and solution techniques in various mathematical and computational contexts. Furthermore, we culminate our discussion by illustrating the practical relevance of the Banach Fixed Point Theorem through a concrete example, demonstrating its applicability in solving real-world problems and providing insight into its computational implications.

\section{Complete metric spaces}
%It can be interesting to have a way to always measure how elements are a part one another given a certain set.
Metrics and metric spaces provide a structured framework for precisely measuring the relationship between elements within a set. This facilitates deeper insights into their interconnections and dependencies, enabling pattern recognition, association identification, and informed decision-making across diverse domains. In this section, we delve into complete metric spaces, which offer a comprehensive framework for understanding the entirety of a set's elements and their interrelations.

\begin{defn}[Metric and metric space]
	A metric on a set $M$ is a function $ d : M\times M \rightarrow \mathbb{R}$, such that  for $x, y, z \in M$:
	\begin{itemize}
		\item[(\textit{i})] [Non-negativity] $d(x,y) \geq  0 $   and $d(x,x) =0 $        
		\item[(\textit{ii})] [Identity of indiscernibles] $d(x,y) = 0  \implies x=y$ 
		\item[(\textit{iii})] [Symmetry] $d(x,y) = d(y,x)$   
		\item[(\textit{iv})] [Triangle Inequality] $d(x,y) \leq d(x,z)+d(z,y)$  
	\end{itemize}
	The pair $(M,d)$ is called a metric space.
\end{defn}
Let's now give some interesting examples of metric spaces.
\begin{exa}\phantom{c}
	\begin{itemize}
		\item[1)] Let $M=\mathbb{R}^n$ be the set of real n-uplets. For $x=(x_1,x_2,\cdots,x_n)$ , $y=(y_1,y_2,\cdots,y_n)$ and $z=(z_1,z_2,\cdots,z_n)$ in $ \mathbb{R}^n $ let's define :
		\[d_p(x,y) = \Big(\sum_{i=1}^n|x_i-y_i|^p\Big)^{1/p}~~~\text{with}~~p\geq 1 \]
		The conditions (\textit{i}), (\textit{ii}) and (\textit{iii}) hold due to the properties of the absolute value and the triangle inequality also holds because it is the result of the Minkowsky's inequality \cite{kantorovich2016functional,kumaresan2005topology}. Indeed assuming that $a_i = x_i-z_i$ and $b_i = z_i-y_i$. We can write the triangle inequality as follows :
		\begin{equation}
			\Big(\sum_{i=1}^n|a_i+b_i|^p\Big)^{1/p} \leq \Big(\sum_{i=1}^n |a_i|^p\Big)^{1/p} + \Big(\sum_{i=1}^n |b_i|^p\Big)^{1/p} \label{ToBeProven2}
		\end{equation}
		So, $\Big(\mathbb{R}^n,d_p\Big)$ is a metric space. When $p=1$ we get the manhattan distance, if $p=2$ we have the classical euclidian distance. But one case which is for a particular interest in this work is the case when $p\to \infty$. We get the infinity distance :
		\begin{equation}
			d_{\infty}(x,y) = \underset{1\leq i\leq n}{\max}~|x_i-y_i|
		\end{equation}
		\item[2)] Let $S$ be any non-empty set and $B(S)$, the space of all bounded real-valued functions on $S$. Consider $f$ and $g$ in $B(S)$, and define :
		\begin{equation}\label{MetricBoundedFunc}
			d(f,g) = \underset{x\in S}{sup}~|f(x)-g(x)|
		\end{equation}
		The pair $\Big(B(S), d\Big)$ is a metric space. This metric is called the \textit{supremum metric} or \textbf{uniform metric} \cite{kumaresan2005topology}.
	\end{itemize}
\end{exa}

Having presented some interesting examples of metric spaces, we give some definitions and propositions, necessary to introduce the completeness of metric spaces.
\begin{defn}[Sequences]
	A sequence in a metric space $(X,d)$ is a function $u:\mathbb{N}\to X$. It can be written $\{u_n\}_{n\in\mathbb{N}}$ or just $\{u_n\}$.
\end{defn}
\begin{defn}[Convergence]
	A sequence  $\{x_n\}$ of points in $(X,d)$ is said to be convergent if there is a point $x^*\in X$ such that, for any $\varepsilon>0$, there exists a positive integer $N_0$ such that $d(x_n,x^*)<\varepsilon$ whenever $n>N_0$. The point $x^*$ is termed the \textit{limit of the sequence} and we write $x_n\rightarrow x^*$.
\end{defn}

\begin{pro}
	The limit of a sequence in a metric space is unique.
\end{pro}

\begin{proof}
	The proof of that can be seen in \cite{shirali2005metric}. It leverages the property that metric spaces belong to the family of $T_2$ (or Hausdorff) topological spaces 
\end{proof}

Cauchy sequences play a fundamental role in the study of completeness within metric spaces. A sequence is considered Cauchy if, for any arbitrarily small positive distance, there exists a point in the sequence beyond which all subsequent points are within that distance from each other. Completeness, on the other hand, refers to the property of a metric space wherein every Cauchy sequence converges to a limit within the space itself. Together, Cauchy sequences and completeness provide essential tools for understanding the convergence behavior and structural integrity of metric spaces.

\begin{defn}[Cauchy sequences]
	A sequence $\{x_n\}$ in a metric space $\Big(X,d\Big)$ is called a \textbf{Cauchy sequence} if for any given $\varepsilon > 0$, we can find $N_0\in \mathbb{N}$ such that whenever $\min\{m,n\}>N_0$ we have $$d(x_m,x_n)<\varepsilon.$$ 
\end{defn}

While every convergent sequence is a Cauchy sequence, the reverse of this is only true in `complete metric spaces'.

\begin{defn}[Completeness]
	 The metric space $\Big(X,d\Big)$ is said \textbf{complete} if every Cauchy sequence converges in $X$.
\end{defn}

\begin{exa}
	Using the standard metric (absolute value), $\mathbb{R}$ is complete. Using also the associated standard metric, $\mathbb{R}^n$ is complete \cite{kumaresan2005topology}.  
\end{exa}

\begin{defn}[Uniformly Cauchy sequence of functions]
	We say that a sequence $ f_n : X \rightarrow \mathbb{R}$ is uniformly Cauchy if for a given $\varepsilon>0$, there exists $N \in \mathbb{N}$ with the following
	property:
	\[|f_n(x)-f_m(x)|< \varepsilon ~~\text{for all} ~~ x\in X~~ \text{and for all}~~m,n\geq N.\]
\end{defn}

Normed spaces are a subset of metric spaces, where vectors have associated norms. They are useful because they provide a framework for measuring distances and magnitudes in vector spaces, aiding in the analysis of mathematical structures.

\begin{defn}[Norms and normed spaces]
	Let's $V$ be a linear space over a field $\mathbb{K}$. A \textbf{norm} on $V$ is a map $ ||\cdot|| : V\times V\to \mathbb{R}^{+} $ such that for any $x,y \in V$ and $\alpha \in \mathbb{K}$ :
	\begin{itemize}
		\item[(i)] [Positivity] $~~~~~~~~~~~~~||x||>0$ for every $x\neq 0$
		\item[(ii)] [Homogeneity] $~~~~~~~~||\alpha x|| = |\alpha|||x||$
		\item[(iii)] [Triangle inequality] $~~||x+y||\leq ||x||+||y||$
	\end{itemize} 
\end{defn}
The pair $\Big(V, ||\cdot ||\Big)$ is called a \textbf{normed space}.
\begin{rem}
While normed spaces focus on vector spaces and their associated norms, metric spaces provide a broader context for studying distances and relationships between points in arbitrary sets. 
	If $||\cdot||$ is a norm, we can get a metric $d$ from it by doing $d(x,y) = ||x-y||$. This is called an \textbf{induced metric}.
\end{rem}
%From this definition (Cauchy Sequence), we can directly notice that
%\begin{exa}
%	\textcolor{red}{SIMPLE EXAMPLE}
%\end{exa}
\begin{exa}\label{ExampleBX}
	Let $X$ be any nonempty set. Let us consider the linear space $B(X)$ of all \textit{bounded real valued functions} on $X$ under the \textit{sup norm} $||\cdot||_\infty$. In $B(X)$, a sequence $\{f_n\} $ is Cauchy if and only if it is \textit{uniformly Cauchy}.
\end{exa}
\begin{proof}
	Detailed proof in \cite{amath731}.
\end{proof}

Banach spaces are a special class of normed vector spaces that are also complete metric spaces. This property of completeness ensures that every Cauchy sequence in the space will converge to a limit within the space itself, making Banach spaces crucial in various areas of analysis and functional analysis.
Let now define the concept of \textbf{Banach Space} which will be important in the last chapter.

\begin{defn}
	A Banach space is a complete normed space.
\end{defn}

\begin{exa}\label{Exa:Banach}
	The normed space $\left(B(X),||\cdot||_\infty\right)$ as defined in the Example \ref{ExampleBX} is a Banach space.
\end{exa}

\begin{proof}
	The proof can be found in \cite{kumaresan2005topology}. We here just sketch it.\newline
	Let $\{f_n\}\in B(X)$ be a Cauchy sequence. By Example \ref{ExampleBX}, the sequence $\{f_n\}$ is uniformly Cauchy.
	
	Let's fix $x\in X$ and consider the sequence of scalars $\{f_n(x)\}$. Since $\mathbb{R}$ is complete, the sequence converges to a real number $f(x)\in \mathbb{R}$.\newline
	%	, to show the dependance of $\alpha=\lim_{n\to\infty}f_n(x)$ on $x$.
	First, we have to show that $f$ is bounded on $X$ and,\newline 
	Secondly we have to show that $\{f_n\}$ converges uniformly to $f$ on $X$. That will complete the proof.
\end{proof}
%\textbf{Nota :} In this proof, we are using the continuity of the distance function.
Metric spaces are not always complete. We here present a way to make complete an incomplete metric space.
\subsection{Completion of Metric Spaces \cite{kumaresan2005topology,shirali2005metric}}
Let $(X, d)$ be a non-complete metric space. It is always possible to complete the space into a larger space in such a way that every Cauchy
sequence in X has a limit in the completion. To do this, we should add new points to $(X, d)$ and extend $d$ to all these new points so that each non-convergent Cauchy sequences in $X$ find limits among these new points.
% Let $(X, d)$ be a metric space that is not complete. It is always possible to construct a larger space which is complete and contains just enough points so that every Cauchy
% sequence in X has a limit in the larger space. In fact, we need to adjoin new points to
% $(X, d)$ and extend $d$ to all these new points in such a way that the formerly nonconvergent Cauchy sequences find limits among these new points and the new points are limits of sequences in $X$ \cite{shirali2005metric}.

\begin{defn}[Completion]
	Let $(X, d)$ be a metric space. A complete metric space $(X^{*}, d^{*})$ is considered as a completion of the metric space $(X, d)$ if :
	\begin{itemize}
		\item[(i)] $X$ is a subspace of $X^{*}$
		\item[(ii)] Every point of $X^{*}$ is the limit of some sequence in $X$
	\end{itemize}
\end{defn}

\begin{thm}
	Let $(X,d)$ be a metric space. Then there exists a completion of $(X,d)$.
\end{thm}
\begin{proof}
	For the proof, see \cite{kumaresan2005topology}.
\end{proof}
%\begin{proof}
%	Let $(X, d)$ be any metric space. Fix a point $o\in X$. For each $x\in X$, consider the function $f_x(y) := d(y,x)-d(y,o)$. Show that $f_x\in B(X)$ and that the map $\varphi : x \mapsto f_x$ is an isometry of X into $\left(B(X),||\cdot||_\infty\right)$. If we let $Y$ to be closure of $\varphi(X)$ in $B(X)$, then $(Y,d_\infty)$ is a completion of $(X, d)$ via the isometry $\varphi(X)$ of $X$
%	into $Y$. (Here $d_\infty$ denotes the restriction of the metric on $B(X)$ to $Y$.)
%	
%	Clearly now $f_x \in B(X)$ :
%	\[|f_x(y)|\leq |d(x,y)- f(y,o)|\leq d(x,o)~~~\forall y\in X.\]
%	Hence $||f_x||_{\infty}\leq d(x,o)$. To show that the map $\varphi$ is a isometry, we need to prove that for any $x,y \in X$,
%	\[d(x,y) = d(\varphi(x),\varphi(y)) =  \underset{z\in X}{sup}|f_x(z)-f_y(x)|.\]
%	We have, for any $x\in X$,
%	\begin{align*}
%		|f_x(z)-f_y(z)| & = |d(x,z)-d(z,o) - [d(y,z)-d(z,o)]|\\
%		& = |d(x,z)-d(y,z)|\\
%		& \leq d(x,y).
%	\end{align*}
%	Also, when $z=y$, we find that $f_x(y)-f_y(y)= d(x,y)$. Thus $||f_x-f_y||=d(x,y)$. Thus $\varphi$ is isometry of $X$ into $B(X)$. If we take $Y$ to be the closure of $\varphi(X)$ in $B(x)$, then $Y$ is complete, being a closed subset of the complete metric space $B(X)$. By our very consytruction, $\varphi(X)$ id dense in $Y$. Thus we have a completion of $(X,d)$.
%\end{proof}

Let us now present the contraction mapping and the Banach Contraction Principle. Those are the core concepts that we will use throughout the work.
\section{Contraction mapping and fixed-point}
\begin{defn}
	Let $X$ be an nonempty set and $f:X\to X$ be a mapping on that set :
	\begin{itemize}
		\item[-] A point $x$ is said to be a \textbf{fixed point} of $f$ if $f(x)=x$.
		\item[-] We will write $Fix(f) = \{x\in X : f(x) = x\}$ the set of fixed points of $f$ on $X$.
	\end{itemize}
\end{defn}
\begin{pro}\label{prop : propFix}
	Let $X$ be a nonempty set and $f:X\to X$ a mapping defined on it. If $x\in X$ is a unique fixed point of $f^n$ with $f^n = \underset{n-times}{\underbrace{f\circ f\circ\cdots\circ f}}$ for any $n>1$, then it is the unique fixed point of $f$ and reversely :
	\[Fix(f^n) = \{x\} \iff Fix(f) = \{x\}\]
\end{pro}
\begin{proof}
	Let $x$ be a fixed point for $f^n$ in $X$. Then we can write :
	\[f(x)=f\left(f^n(x)\right)=f^{n+1}(x) = f^n\left(f(x)\right)\implies ~~f(x) ~~\text{is a fixed point of}~~ f^n\]
	Since $x$ is the only one fixed point for $f^n$, then $f(x)=x$, i.e. $x$ is also a fixed point for $f$. Let's now show that the aforementioned fixed point is unique also for $f$ :\newline
	Let's take $y\neq x$ another fixed point of $f$, and compute
	$y=f(y)=f^2(y) = \cdots =f^n(y)$ or $f^n(y)=y\iff y=x$. Then, we've clearly seen that $x$ is the unique fixed point for $f$.
\end{proof}
Before stating the Banach Contraction Principle, let's define all the important notions related to that specific theorem.
\begin{defn}\cite{debnath2021metric, gopal2017background,naylor1982linear,pata2019fixed}
	Let $(X,d)$ be a metric space and $f:X\to X$ a mapping on $X$. $f$ is said to be \textit{Lipschitz} continuous if there exist $\alpha>0$ such that :
	\begin{equation}
		d\left(f(x),f(y)\right)\leq \alpha\cdot d(x,y)~~~ \forall x,y \in X
	\end{equation}
	\begin{itemize}
		\item[-] If $\alpha\in [0,1)$, $f$ is said to be a \textit{contraction}.
		\item[-] If $\alpha=1$ then $f$ is said \textit{non-expansive}.
		\item[-] If $d\left(f(x),f(y)\right)<d(x,y), ~ \forall x\neq y$ then $f$ is contractive.
	\end{itemize}
\end{defn}
\begin{pro}
	Let $(X,d)$ be a metric space and $f:X\to X$ a contraction mapping with $\alpha \in (0,1)$. If $f$ has a fixed point, that point is unique.
\end{pro}
\begin{proof}
	Suppose we have two fixed points $x$ and $y$ for $f$ with always $x\neq y$. Because $f$ is a contraction, with $\alpha\in (0,1)$ we can write :
	\[0\neq d(x,y) = d\left(f(x),f(y)\right)\leq \alpha\cdot d(x,y),\]
	which is impossible (contradictory). So, the fixed point is unique.
\end{proof}
As we have now a formal definition of fixed point, contractions and some obvious implications, let's now state the Banach Contraction Principle.
\begin{thm}[Banach Contraction Principle : BCP \cite{fixedpoint}]\label{thm:BCP}
	Let $(X,d)$ be a complete metric space, $f:X\to X$ a contraction. Then $f$ has a unique fixed point $x^*$ and for each $x \in X$, $\lim\limits_{n\to\infty}f^n(x)=x^*$. Moreover, $$d\left(f^n(x),x^*\right)\leq \dfrac{\alpha^n}{1-\alpha}d\left(x,f(x)\right).$$.
\end{thm}
\begin{proof}\label{proofBCP}
	Let us construct a sequence $\{x_n\}$, starting on $x_0$ and using the following recurrence :
	\[x_{n} = f(x_{n-1}) ~~ \forall n\in \mathbb{N}~~~ \text{i.e.}~~~x_n = f^n(x_0)\]
	We show that $\{x_n\}$ is a Cauchy-Sequence. In fact, since $f$ is a contraction, we should have :
	\begin{align}\label{eq : demons1}
		d(x_{m+1},x_{m}) &= d\left(f(x_{m}),f(x_{m-1})\right)\leq \alpha\cdot d(x_m,x_{m-1})\nonumber\\
		& = \alpha\cdot d\left(f(x_{m-1}),f(x_{m-2})\right)\leq \alpha^2\cdot d(x_{m-1},x_{m-2})\nonumber\\
		& \vdots _{ }\nonumber\\
		& = \alpha^{m-1}\cdot d\left(f(x_{1}),f(x_{0})\right)\leq \alpha^m\cdot d(x_{1},x_{0})
	\end{align}
	Now, for any given $m$ and $n$, positive and with $n<m$, the triangle inequality applied recursively gives :
	\begin{align}\label{eq : demons2}
		d(x_m,x_n) &\leq d(x_m,x_{m-1}) + d(x_{m-1},x_{m-2})+\cdots + d(x_{n-1},x_{n})\nonumber\\
		&\leq \left(\alpha^{m-1}+\alpha^{m-2}+\cdots +\alpha^{n}\right)\cdot d(x_1,x_0) ~~~ \text{from \ref{eq : demons1}}\nonumber\\
		&\leq \alpha^n\left(\alpha^{m-n-1}+\alpha^{m-n-2}+\cdots +1\right)\cdot d(x_1,x_0) \nonumber\\
		&\leq \frac{\alpha^n}{1-\alpha}\cdot d(x_1,x_0)
	\end{align}
	As we can now see from $\lim\limits_{n\to\infty}\alpha^n = 0$ and $d(x_1,x_0)$ is fixed, we can say $d(x_m,x_n)\to 0$ as $m,n\to \infty$. So, $\{x_n\}$ is a Cauchy sequence and since $X$ is complete, there exits only one $x^*\in X$ such that $x_n\to x^*$. So, by continuity of contractions,  we can write :
	\[x^* = \lim\limits_{n\to\infty}x_{n+1} = \lim\limits_{n\to\infty}f(x_n) = f\left(\lim_{n \to \infty}x_n\right) = f\left(\lim_{n \to \infty}f^n(x_0)\right) = f(x^*)\]
	And the proposition above (Proposition \ref{prop : propFix}), we have a way to find this unique fixed point independently on the starting point.\newline
	And to show a good estimate of how we are approaching the limit we can rewrite the Expression \ref{eq : demons2} like this :
	\[d\left(f^n(x),x^*\right)\leq \dfrac{\alpha^n}{1-\alpha}d\left(x,f(x)\right) ~~~\equiv~~~ d\left(x_n,x^{*}\right)\leq \dfrac{\alpha^n}{1-\alpha}d\left(x_0,x_1\right)\]
\end{proof}
\begin{rem}
	The beauty of the Banach Contraction Principle is that it only require completeness for the metric space and contraction property on the mapping. Also, we can get closer and closer to the fixed point, starting from any point $x\equiv x_0\in X$ by just applying recursively the mapping on that specific point.
\end{rem}
\section{Application of the Banach Contraction Principle - An example}
The Banach Fixed Point principle or Banach Contraction Principle (BCP) is very important in many fields, in mathematics as well as in its applications \cite{ansari2023fixed, mannan2021study, amath731, zeidler2012applied}. We here introduce a problem whose solution is guaranteed by the \textit{Picard-Lindelof} Theorem, the uniqueness of solutions of a first order Ordinary Differential Equation (ODE). We show that its solution can be obtained  using the BCP.\newline
Let's say we have the following initial value problem (which is here an ODE with initial conditions) :
\begin{equation}
	y' = F(x,y)~~~;~~~ y(x_0) = y_0~~~ \text{with}~~ y' = \dfrac{dy}{dx}\label{EquaDiff}
\end{equation}
This problem is in fact equivalent to the following :\newline
\begin{align}
	y(x)-y(x_0) &= \int_{x_0}^{x}F\Big(t,y(t)\Big)dt~~\text{as we have $y_0=y(x_0)$}\nonumber\\
	\Rightarrow y(x) &= y_0 + \int_{x_0}^{x}F\Big(t,y(t)\Big)dt\label{EquivalentEqDif}
\end{align}
So, the two problems \ref{EquaDiff} and \ref{EquivalentEqDif} are exactly the same. And from this last Expression \ref{EquivalentEqDif} we can define a sequence given by :
\begin{equation}
	y_{n+1}(x) = y_0 + \int_{x_0}^{x}F\Big(t,y_{n}(t)\Big)dt\label{EqDifRec} ~~~ \text{with $y_0 = y_n(x_0)~~\forall~ n$.}
\end{equation}
The problem can be seen as demonstrating the \textit{Picard-Lindelof Theorem} which states the following :

\begin{thm}[Picard-Lindelof Theorem \cite{noupelah2023ordinary}]
	Let $D \subset \mathbb{R} \times \mathbb{R}^n$ be an open set, $F: D \to \mathbb{R}^n$ a continuous function and $(x_0, \mathbf{y}_0) \in D$. Then the initial value problem
	
	\[
	\frac{d\mathbf{y}}{dx} = y' = f(x, \mathbf{y}), \quad \mathbf{y}(x_0) = \mathbf{y}_0
	\]
	
	has a unique solution $\mathbf{y}(x)$ on closed intervals $I=[x_0-\varepsilon,x_0+\varepsilon]$ with $\varepsilon>0$.
\end{thm}
To easily show the usefulness of the BCP for the problem \ref{EquaDiff}, we can enunciate the \textit{Picard-Lindelof Theorem} in especially the context of that problem as follows and then prove the existence and uniqueness of the solution :
\begin{pro} \cite{zeidler2012applied}\label{propPicard}
        ~~~Assume that :
	\begin{enumerate}
		\item With $S$ a rectangle domain defined as follows
		\[ S = \{(x,w) \in \mathbb{R}^2 : |x - x_0| \leq a ~~\text{and}~~ |w-y_0|\leq b\}  ~~ \text{with $a,b\in \mathbb{R}$ and $w = y(x)$}\]
		The function $F:S\to \mathbb{R}$, given in the ODE \ref{EquaDiff} above, is \textbf{continuous}, and there exists $c\in \mathbb{R}^+ $ such that $|F(x,y)|\leq c.$
		\item $F(x,y)$ \textbf{satisfies the Lipschitz condition} with respect to $y$ on $S$, i.e. there exists $L\geq 0$ such that :
		\[|F(x,y_1(x))-F(x,y_2(x))|\leq L\cdot|y_1(x)-y_2(x)|,~~~\text{for all}~~ (x,y_1),(x,y_2)\in S \]
		\item The real number $h$ satisfies : $0<h\leq a,~~ hc\leq b,~~ hL\leq 1$
	\end{enumerate}
	The following holds :
	\begin{itemize}
		\item[(\textit{i})] The sequence $\{y_n\}$ constructed in Relation \ref{EqDifRec} converges to a certain function $y_{*}$.
		\item[(\textit{ii})] The initial value problem stated in \ref{EquaDiff} and equivalently in its integrale form given by the Expression \ref{EquivalentEqDif} has a unique solution which is exactly $y_{*}$.
		\item[(\textit{iii})] For $n\in \mathbb{N}$ and $k := hL$ we have the following estimates :
		\begin{itemize}
			\item[$\bullet$] $||y_n-y_{*}||_{\infty} \leq \dfrac{k^n}{1-k}\cdot||y_1-y_{0}||_{\infty}$ see the last Expression \ref{eq : demons2}.
			\item[$\bullet$] $||y_{n+1}-y_{*}||_{\infty} \leq \dfrac{k}{1-k}\cdot||y_{n+1}-y_{n}||_{\infty}$
		\end{itemize}
		% With , and as seen in the last assumption, $hL\leq 1\implies k\leq 1$.
	\end{itemize}
\end{pro}
We will not prove this proposition explicitly. We will follow a process, inspired with the assumptions made in the Proposition \ref{propPicard} and that will help us to demonstrate immediately $(\textit{i})$ and $(\textit{ii})$. The point $(\textit{iii})$ about the estimates will follow immediately, due to the Expression \ref{eq : demons2}.
\begin{proof}
	Let $R$ be the rectangular domain :
	\[ R = \{(x,w) \in \mathbb{R}^2 : |x - x_0| \leq h ~~\text{and}~~ |w-y_0|\leq h.c\}  ~~ \text{with $h,c\in \mathbb{R}$ and $w = y(x)$}\]
	Following the assumptions made for this Proposition, we know that $0<h\leq a,~~ hc\leq b$, i.e. $R\subset S$.\newline
	Let $X$ be the set of all real-valued continuous functions $y = y(x)$ on $[x_0-h,x_0+h]$. Then $X$ is a closed subset of the normed space $\mathcal{C}\left([x_0-h,x_0+h]\right)$ with the sup norm $||\cdot||_\infty$.\newline
	Let's now define an operator (a functional) like this (inspired by the second formulation of our problem, given in \ref{EquivalentEqDif}) :
	\begin{align}
		T : X\to & X \nonumber\\
		g(x) \mapsto & h(x) = Tg = y_0 + \int_{x_0}^{x}F\Big(t,g(t)\Big)dt\label{Operator}
	\end{align}
Since $$d\Big(h(x),y_0\Big)=\underset{x}{\sup}{\left| \int_{x_0}^{x}F\Big(t,g(t)\Big)dt \right|} = \left|\left| \int_{x_0}^{x}F\Big(t,g(t)\Big)dt \right|\right|_{\infty} \leq c(x-x_0)\leq c\cdot h \leq b,$$
we can say that the operator $T$ in \ref{Operator} is well defined.\newline
Now, we have to prove that this operator is a contraction mapping. For that, let's take $g,g_1\in X$. We'll have :
\begin{align}
	d\left(Tg,Tg_1\right) = d(h,h_1) = &~~ \left|\left|\int_{x_0}^{x}\Big(F\left(t,g(t)\right)-F\left(t,g_1(t)\right)\Big)dt\right|\right|_{\infty}\nonumber\\
	\leq & ~~ \int_{x_0}^{x}\left|\left|\Big(F\left(t,g(t)\right)-F\left(t,g_1(t)\right)\Big)\right|\right|_{\infty}dt\nonumber\\
	\leq &~~ L\cdot\int_{x_0}^{x}\left|g(t)-g_1(t)\right|dt\text{~~~using the second assumption.}\nonumber\\
	\leq &~~ Lh\cdot d\left(g,g_1\right) = k\cdot d\left(g,g_1\right) \label{Contraction}.
\end{align}
In \ref{Contraction}, from the third assumption, we know that $k<1$, then the operator $T$ is a contraction.\newline
Hence now $T$ is a contraction mapping and X is a closed subset of a metric space, $X$ is also a complete metric space and by the Banach Fixed Point principle, $T$ has one unique fixed point $g$ which is the exactly the unique solution to the given ODE.
\end{proof}
\begin{exa}\label{ExampleEQUADIFF}
    Consider the following first order initial value problem 
    \[x'(t) = \frac{1}{2}x(t)-t, ~~~~ x(0) = 0.\]
\end{exa}
From the Picard-Lindelof Theorem above, we are sure that we'll get a solution, and to converge to it, we started with a random function and used an iterative approach, as suggested by the Banach contraction principle that we used to prove the theorem (through the Proposition \ref{propPicard}). We obtained the following results :
% \begin{figure}[htbp!]
%     \centering
%     \includegraphics[width=0.9\textwidth]{Figures/EDOwithBCP.png}
%     \captionof{figure}{Numerical solution to $~x'(t)=1/2.x(t)-t~$ applying insights from the Banach Fixed Point Theorem}\label{fig:EDOwithBCP} 
% \end{figure}
\begin{center}
    \includegraphics[width=0.9\textwidth]{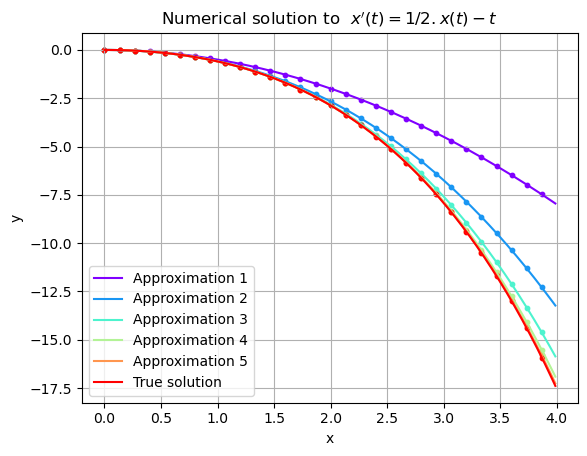}
    \captionof{figure}[Numerical solution to $~x'(t)=1/2.x(t)-t~$ using the Banach Contraction Principle]{Numerical solution to $~x'(t)=1/2.x(t)-t~$ applying insights from the Banach Fixed Point Theorem}\label{fig:EDOwithBCP} 
\end{center}
In the Figure \ref{fig:EDOwithBCP} we can see that the algorithm converges quickly to the true function (just after 5 iterations on $[0,4]$. To make it work, we implemented the operator \ref{Operator} and used the same kind of iterative process as suggested by the Banach contraction principle. The full implementation we did, as well as the use of other methods for validation and stability, can be seen in our Github \cite{MyGithub}. The Example \ref{ExampleEQUADIFF} has also been treated in \cite{haciouglu2021iterative} and the interested reader can compare our results with those found in that article.

\section*{Chapter Conclusion}
% In this chapter, we have discussed about contraction mappings in a metric space and the Banach Contraction Principle (or the Banach Fixed Point Theorem). We have also shown how the theorem can be used to prove the existence and uniqueness of solutions for first order Ordinary Differential Equations with initial conditions. In the next chapter now, we will present an overview of Reinforcement Learning and this will help us to establish the connection with the theories introduced here.
In this chapter we have discussed contraction mappings in a metric space and the Banach contraction principle (or Banach fixed point theorem). We have also shown how the theorem can be used to prove the existence and uniqueness of solutions to first-order ordinary differential equations with initial conditions, and even to find the solution. In the next chapter we will give an overview of Reinforcement Learning, which will help us to relate it to the theories introduced in this chapter.
\chapter{Overview on Reinforcement Learning}\label{chap2}
In this chapter, we will first formally present the Markov Decision Process framework and how it relates to Reinforcement Learning. Second, we will talk about optimality in Reinforcement Learning, and finally we will discuss some methods that are at the core of how optimality is realised in Reinforcement Learning.
\section{Markov Decision Processes in the Reinforcement Learning setting}
We here present a formal definition of the Markov Decision Process. But before that, we discuss about Markov Chain because it allows us to present the fact that, there are situations in which the previous observation contains all the information about the past, in such a way that we don't need to look at the whole history.
\begin{defn}[Markov chain \cite{lazaric2013markov}]
	Let $\mathcal{S}$ (a subset of $\mathbb{R}^n$ for example) represent the state space. The discrete-time dynamic system $\{S_t\}_{t\in \mathbb{N}}\in \mathcal{S}$ is a Markov chain if :
	\begin{equation}
		\mathbb{P}(S_{t+1}=s|S_t,S_{t-1},\cdots,S_0) = \mathbb{P}(S_{t+1}=s|S_t),
	\end{equation}
i.e. everything we need to predict (in probability) the future is contained in the current state (this is what we call Markov property). Given an initial state $S_0\in \mathcal{S}$, a \textbf{Markov chain} is defined by the following transition probability :
\begin{equation}
	p(s'|s) = Pr(S_{t+1}=s'|S_t=s).
\end{equation}
\end{defn}
We can now clearly define the Markov Decision Process and related concepts, important for Reinforcement Learning.
\begin{defn}[Markov Decision Process \cite{lazaric2013markov, sigaud2013markov}]
	Formally, an Markov Decision Process (MDP) is a 5-uplet $\mathcal{M} = \langle\mathcal{S},\mathcal{A}, p, r,\gamma\rangle$ where :
	\begin{itemize}
		\item $\mathcal{S}$ represents the \textbf{state space}. It can be finite, countable infinite or continuous.
		\item $\mathcal{A}$ represents the \textbf{action space} which could also be finite, countable infinite or continuous.
		\item $p(s,a,s')$ is the \textbf{transition probability} (called  environment dynamics) defined by :
		\begin{equation}
			p(s,a,s') \equiv p(s'|s,a) = Pr(S_{t+1}=s'|S_t=s,A_t=a),
		\end{equation}
		with $s,s'\in \mathcal{S}$ and $a\in\mathcal{A}$.
		This represents the probability of observing a next state $s'$ when the action $a$ is taken in the state $s$.
		\item $r(s,a,s')$ is the \textbf{reward function} called also a \textbf{reinforcement} obtained when taking the action $a$, in the state $s$ and the next state observed is $s'$. So :
		\[r : \mathcal{S}\times\mathcal{A}\times\mathcal{S}\to \mathbb{R}\]
		\item The transition probability $p$ and the reward function $r$ defines the \textbf{environment model}.
		\item $\gamma \in [0,1)$ is the \textbf{discount factor} used to define (let's say for discrete time here for illustration) the expected cumulative reward :
		\[G = \sum_{t=0}^{\infty}\gamma^t\cdot r(S_t,A_t,S_{t+1}).\]
		\end{itemize}
\end{defn}
If we have a problem which can be modeled as an MDP, the goal is generally to maximize the quantity $G$ by choosing at each time and for each state the action $A_t$. That choice is defined by a function $\pi:\mathcal{S}\to \Delta\mathcal{A}$ called \textbf{policy}. Where $\Delta\mathcal{A}$ represents a probability distribution under actions.
\begin{defn}[Policy or decision rule]
	Policy is a mapping of states to actions, which can be :
	\begin{itemize}
		\item Deterministic : in the same conditions, the same action is always chosen.
		\item Stochastic : at each time, actions are only chosen with a certain probability. In this case the policy is a mapping from the state space to a probability distribution over actions.
	\end{itemize}
	The policy can also be varying as the time evolve (Non-stationary policy) or constant (Stationary policy).
\end{defn}
We have given the foundational definitions, let us now reformulate everything in the Reinforcement Learning setting : 
\begin{enumerate}
	\item \textbf{The agent-environment interface}\newline
	Markov Decision Processes (MDPs) are indeed a simple way to frame the problem of learning from interaction to achieve a goal. We are here define some important terms in the field of Reinforcement Learning :
	\begin{itemize}
		\item The element that has to learn and make decisions is called \textbf{agent}.
		\item Everything with which the agent interacts, including any entity outside the agent, is called the \textbf{environment}.
        \item The agent and the environment are in constant interaction, the former selecting actions, the latter responding to those actions and presenting new situations to the agent. These new situations are called \textbf{states} of the environment.
		\item The agent also receives \textbf{rewards} (special numerical values) from the environment and tries to maximise these over time through its choice of actions.
		\item We call \textbf{action} the way the agent chooses to interact with its environment. Or simply, the interaction of the agent with its environment.
	\end{itemize}
	Let us now call $S_t$, $A_t$ and $R_t$ respectively the state, the action and the reward at a specific time-step $t$.The MDP and the agent will now produce the following kind of sequences :
	\[S_0,A_0,R_1~~,~~S_1,A_1,R_2~~,~~S_2,A_2,R_3~~,...\]
	The set of rewards, actions and states will be respectively written as $\mathcal{R},\mathcal{A}$ and $ \mathcal{S}$. We'll also call $\mathcal{A}(s)$ the actions which are feasible in a specific state $S_{t}=s$.\newline

	The Figure \ref{fig:agentSuttonBarto} shows schematically all the above concepts and how they are related.
	\begin{center}
		\includegraphics[scale=0.5]{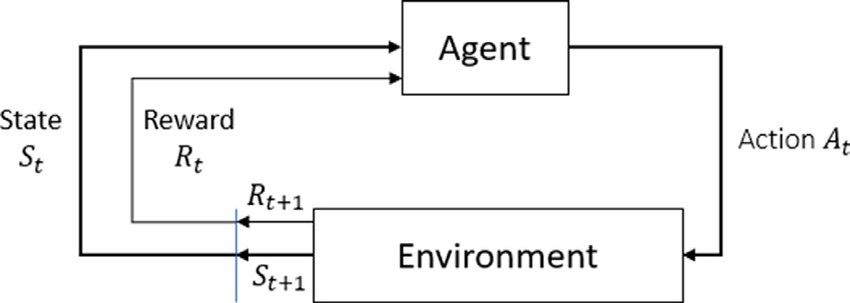}
		\captionof{figure}[Agent-environment interaction in an MDP]{Agent-environment interaction in an MDP \cite{sutton2018reinforcement}}\label{fig:agentSuttonBarto}
	\end{center}
	\item \textbf{Model of the environment} \label{section : model_of_env}\newline
	Let us now suppose that we are able to have access to all the dynamic of our environment.  Then, we can generally define a function which gives the probability of transiting in a certain state $S_{t}$ and receive a certain reward $R_{t}$ if we were previously in the state $S_{t-1}$ and made the action $A_{t-1}$. That kind of function, defined for the whole action and state spaces is called the \textbf{model of the environment}. It can be generally defined as follows :
	\begin{equation}\label{equation : dynamic-environment}
		p(s',r|s,a) \equiv Pr\Big(\{S_t = s', R_t = r | S_{t-1}=s, A_{t-1}=a \}\Big)~\forall s,s' \in \mathcal{S} ~, r\in \mathcal{R}; ~ a\in \mathcal{A}(s)
	\end{equation}
	Most of the time, it is assumed that, the next values depend only on the immediate previous values (\textbf{Markov property}) \cite{lazaric2013markov}. The Markov Property can be formalized as follows :
	\[p(s',r|s,a)  = Pr\Big(\{S_t, R_t | S_{t-1},A_{t-1}, S_{t-2},A_{t-2},...  \}\Big)= Pr\Big(\{S_t = s', R_t = r | S_{t-1}=s, A_{t-1}=a \}\Big)\]
	This function defines the dynamic of a Markov decision process.
	\item \textbf{Other important definitions from the model of the environment}\newline
	MDPs are powerful tools for modeling and solving problems where decision making involves uncertainty and sequential actions. We are here giving other formal definitions which shows how we can derive well-known concepts from the general Definition given by the Expression \ref{equation : dynamic-environment} \cite{sutton2018reinforcement}.
	\begin{enumerate}
		\item Second form of the environment model\newline
		This definition is the commonly used but is not as general as the one we gave in Relation \ref{equation : dynamic-environment}. The reward is not explicitly considered :
		\begin{equation}\label{equation : dynamic-environment-noReward}
			p(s'|s,a) \equiv Pr\left\{S_t = s' | S_{t-1}=s, A_{t-1}=a \right\} = \sum_{r\in\mathcal{R}}p(s',r|s,a)
		\end{equation}
		\item Definition of the reward\newline
		Here we define the reward without considering the next state of the environment:
		\begin{equation}\label{equation : reward-definition}
			r(s,a) \equiv \mathbb{E}\left[R_t|S_{t-1}=s,A_{t-1}=a\right] = \sum_{r\in\mathcal{R}}\left(r\cdot\sum_{s'\in\mathcal{S}}p(s',r|s,a)\right)
		\end{equation}
		\item Reward depending on next state\newline
		Here we consider the next state in our definition of reward. That captures the most similar to our usual view of the reward function in real life:
		\begin{equation}\label{equation : reward-definition-nextState}
			r(s,a,s') \equiv \mathbb{E}\left[R_t|S_{t-1}=s,A_{t-1}=a, S_{t}=s'\right] = \sum_{r\in\mathcal{R}}\left(r\cdot\dfrac{p(s',r|s,a)}{p(s'|s,a)}\right)
		\end{equation}
	\end{enumerate}
	\item \textbf{Formalization of the concepts of goals, rewards and returns}\newline
	In Reinforcement Learning (RL), the purpose or goal of the agent is formalized in terms of a special signal called a ``reward" as explained previously, which is sent to the agent from the environment. In other words, given the reward at any given time as just a real number, the ``goal" of the agent is to maximize the total amount of rewards it receives. This is not a maximization of the immediate reward, but of the cumulative reward in the long run. With this in mind, if we call $R_t$ the reward we get at time $t$, we can generally define a function that we'll call \textbf{return} at that time like this :
	\begin{equation}\label{equation : return-def}
		G_t = R_{t+1}+\gamma\cdot R_{t+2}+\gamma^2\cdot R_{t+3}+\cdots = \sum_{k=0}^{\infty}\gamma^{k}\cdot R_{t+k+1} ~~~ \text{with}~~ \gamma\in[0,1)
	\end{equation}
	In this Expression \ref{equation : return-def}, the \textbf{discount factor} $\gamma$ can also be taken equal to one but this will lead to the divergence of the return function in infinite cases. This definition is mathematically consistent, and it is also consistent in real-life. In fact :
	\begin{itemize}
		\item \underline{Mathematically} : Taking $\gamma$ in the interval $[0,1)$ ensures the convergence of the series if only the reward function is bounded.
		\item  \underline{In real life} : This definition of the return function is consistent with the concept that immediate rewards are more appreciated than rewards deferred into the future.
	\end{itemize}
	
	\item \textbf{Policies and Value Functions}\newline
	Almost all RL algorithms involve estimating \textbf{value functions}, which are functions of states (or of state-action pairs) that estimate how good it is for an agent to be in a particular state (or to perform a particular action in a particular state).
	
	The reward that an agent can expect to receive in the future depends on which action will be taken. Accordingly, value functions are defined with respect to particular ways of acting, called \textbf{policies}. Formally, a policy is a mapping from states to action space. In general, it quantifies the probability of selecting a certain action in a specific state. So, the policy can be defined as {follows:}
	\begin{equation}\label{equation : policy-definition}
		\pi\left(a|s\right) \equiv \pi\left(A_t=a|S_t=s\right) = Pr\left\{a\in\mathcal{A}(s)|s\right\},
	\end{equation}
	% With this definition of policy, we can redefine other concepts. For example, the definition we gave for the reward in \ref{equation : reward-definition} become (after shifting one step in time for convenience) :
	% \begin{equation}\label{equation : reward-definition}
	%     r(s,a) \equiv \mathbb{E}\left[R_{t+1}|S_{t}=s,A_{t}=a\right] = \sum_{r\in\mathcal{R}}\left(r\cdot\sum_{s'\in\mathcal{S}}p(s',r|s,a)\cdot \pi(a|s)\right)
	% \end{equation}
	where $\mathcal{A}(s)$ represents the set of possible actions when we are in the state $s$.\newline
	The value function now, for a state $s$, under a policy $\pi$, denoted $v_{\pi}(s)$, is the expected payoff of starting in that state and then following that policy. For MDPs we can formally say that the value function is given by :
	\begin{equation}\label{equation : state-value-function}
		v_{\pi}(s) \equiv \mathbb{E}_{\pi}\left[G_t|S_t=s\right] = \mathbb{E}_{\pi}\left[\sum_{k=0}^{\infty}\gamma^{k}\cdot R_{t+k+1}|S_t=s\right]
	\end{equation}
	We can similarly define the value function of taking an action $a$ while following the same policy, starting in the state $s$, which is called the \textit{Q-value} :
	\begin{equation}\label{equation : action-value-function}
		q_{\pi}(s,a) \equiv \mathbb{E}_{\pi}\left[G_t|S_t=s, A_t = a\right] = \mathbb{E}_{\pi}\left[\sum_{k=0}^{\infty}\gamma^{k}\cdot R_{t+k+1}|S_t=s, A_t = a\right]
	\end{equation}
	The equations \ref{equation : state-value-function} and \ref{equation : action-value-function} correspond respectively to the \textbf{state-value function} and the \textbf{action-value function}. Also, the state value function can be seen as the average of all the values that can be achieved by acting from a given state, i.e., we can see the state value function as a weighted average of the action-value functions that we can obtain starting from that state. So, formally :
	\begin{equation}\label{equation : relation-action-and-state-value-function}
		v_{\pi}(s) \equiv \sum_{a\in\mathcal{A}(s)}q_{\pi}(s,a)\cdot\pi(a|s)
	\end{equation}
	Now the following holds between the value of $s$ and the value of its possible successor states for any policy $\pi$ and any state $s$ :
	\begin{align}
		v_{\pi}(s) &= \mathbb{E}_{\pi}\left[G_t|S_t=s\right]\nonumber \\
		&= \mathbb{E}_{\pi}\left[R_{t+1}+\gamma\cdot G_{t+1}|S_t=s\right] \nonumber \\
		&= \sum_{a}\pi(a|s)\left(\sum_{s'}\sum_{r}p(s',r|s,a)\Big[r+\gamma\cdot\mathbb{E}_\pi\left[G_{t+1}|S_{t+1}=s'\right]\Big]\right) \label{in-equation : explanation-in-value-func} \\
		\Rightarrow v_{\pi}(s) &= \sum_{a}\pi(a|s)\left(\sum_{s',r}p(s',r|s,a)\Big[r+\gamma\cdot v_\pi(s')\Big]\right) \\
		& ~\text{using then the equation \ref{equation : reward-definition}}\nonumber\\
		\Rightarrow v_{\pi}(s) &= \sum_{a}\pi(a|s)\cdot\left(r(s,a)+ \gamma\cdot\sum_{s'}p(s'|s,a) v_\pi(s')\right)
		 \label{equation : Bellman-for-state-consistency}
	\end{align}
	The Relation \ref{in-equation : explanation-in-value-func} is explained by the definition of the reward function given by the equation \ref{equation : reward-definition}, and also by the fact that we normally have to do a weighted summation under the action set $\mathcal{A}$ to account for all actions.
	The equation \ref{equation : Bellman-for-state-consistency} is then called \textbf{Bellman equation} for the value function $v_\pi$.\newline
	We can define in the same way the Bellman equation corresponding to the action-value function as follows \cite{TheArtOfRL} :
	\begin{equation}\label{eq:QvalueGen}
		q_\pi(s,a) = r(s,a) + \gamma\cdot\sum_{s'}p(s'|s,a)\cdot\left(\sum_{a'}\pi(a'|s')\cdot q_\pi(s',a')\right)
	\end{equation}
	
\end{enumerate}
From now on we can notice that the MDP framework is a considerable abstraction of the problem of goal-directed learning from interaction.

\section{Optimal Policies and Optimal Value Functions}\label{sec:OptimalPol}
Here we present the optimal value functions and discuss how they can be used to find the optimal policy \cite{sutton2018reinforcement}. In fact, the value function defines a partial ordering over policies. More precisely, given two policies $\pi$ and $\pi'$, we can say,
\begin{equation}
	\pi\geq \pi' \iff v_{\pi}(s)\geq v_{\pi'}(s) ~~~~~ \forall s\in \mathcal{S},\label{PartialOrdering}
\end{equation}
but we are not always able to compare two value functions. So, the policy $\pi$ is defined to be better than or equal to $\pi '$ if its expected return is greater than or equal to that of $\pi '$ for all states.\newline
And now, if $\pi^{*}$ is the optimal policy we have :
\begin{align}
	\left\{
	\begin{array}{lcl}
		v_{*}(s) & \equiv & \underset{\pi}{\max}~ v_{\pi}(s) \\
		q_{*}(s,a)  & \equiv & \underset{\pi}{\max}~ q_{\pi}(s,a) = \mathbb{E}\left[R_{t+1}+\gamma\cdot v_{*}(S_{t+1}) | S_t=s, A_t=a\right]
	\end{array}
	\right.
\end{align}
\subsection{Bellman optimality equation}
Since we take the optimal action at each time (following the optimal policy $\pi^*$), we can write $v_{*}(s)=\underset{a\in\mathcal{A}(s)}{\max}~ q_{\pi^*}(s,a)$ , and this will be the best achievable expected reward at each time. So :
\begin{align}
	v_{*}(s) &=\underset{a\in\mathcal{A}(s)}{\max}~ q_{\pi^*}(s,a)\nonumber \\
	&= \underset{a}{\max}~\mathbb{E}_{\pi^*}\left[G_{t} | S_t=s, A_t=a\right]\nonumber\\
	&= \underset{a}{\max}~\mathbb{E}_{\pi^*}\left[R_{t+1}+\gamma\cdot G_{t+1} | S_t=s, A_t=a\right]\nonumber\\
	&= \underset{a}{\max}~\mathbb{E}_{\pi^*}\left[R_{t+1}+\gamma\cdot v_{*}(S_{t+1}) | S_t=s, A_t=a\right]\nonumber\\
	&= \underset{a}{\max}~\sum_{s',r}p(s',r|s,a)\cdot\left[r+\gamma\cdot v_{*}(s')\right]\nonumber\\
	&= \underset{a}{\max}\left[r(s,a)+\gamma\cdot\sum_{s'}p(s'|s,a) v_{*}(s')\right] \label{equation : Bell-Optim-state-val-func}
\end{align}
The Expression (\ref{equation : Bell-Optim-state-val-func}) follows the definition of the model of the environment given in point \ref{section : model_of_env} on page \pageref{section : model_of_env}. We can then show the consistency with the connection given in Expression \eqref{equation : relation-action-and-state-value-function} by deducing $q_{*}(s,a)$ as follows :
\begin{align}
	q_{*}(s,a) &= \mathbb{E}\left[R_{t+1}+\gamma\cdot\underset{a'}{\max}~ q_{*}(s',a')|S_t=s,A_t=a,~~ S_{t+1}=s', A_{t+1}=a'\right] \nonumber\\
	&= \sum_{s',r}p(s',r|s,a)\cdot\left[r+\gamma\cdot \underset{a'}{\max}~ {q_{*}(s',a')}\right]\nonumber\\
	\Rightarrow q_{*}(s,a) &= r(s,a) + \gamma\cdot\sum_{s'}p(s'|s,a)\cdot \underset{a'}{\max}~ {q_{*}(s',a')}\nonumber\\
	\Rightarrow q_{*}(s,a) &= r(s,a) + \gamma\cdot\sum_{s'}p(s'|s,a)\cdot{v_{*}(s')}
	\label{equation : Bell-Optim-action-val-func}
\end{align}
Equations \eqref{equation : Bell-Optim-state-val-func} and \eqref{equation : Bell-Optim-action-val-func} are called \textbf{Bellman optimality equations}.
For a finite MDP, the Bellman optimality equation has a unique solution (in the case of infinite MDPs, the Bellman optimality equation may not have a unique solution because infinite MDPs can have multiple optimal policies that lead to the same value function).\newline
And having the optimal value function $v_{*}$, one can determine an optimal policy (because the value function defines a partial order on the policy space).

Finally, even if we have a complete and accurate model of the environment (the dynamics), it is usually not possible to compute an optimal policy simply by solving the Bellman equation. Therefore, approximate methods are developed and used \cite{silver2015lecture}.
\section{Solving Markov Decision Processes}
Solving an MDP means try to find an optimal value function $v_{*}$ or $q_{*}$, which corresponds to find the optimal policy $\pi^*$.\newline
There are many ways of doing that, and algorithms were developed for specific challenges. The most used classification consists in distinguishing \textbf{model-based} and \textbf{model-free algorithms}, direct or indirect.\newline
Generally, model-based algorithms concern what is called \textbf{Dynamic Programming} (DP). Model-based algorithms assume that the model of the environment is well known and can be used to find value functions and policies using the Bellman equations.\newline
Model-free algorithms now, is considered as true-RL algorithms. They don't require the availability of a perfect model (we suppose that, the environment is not totally known). This method rely on interaction with the environment. That interaction is done by simulating the policy and generating samples of state transitions and
rewards. Then, the samples are used to estimate value functions and even to improve the used policy.\newline
In model-free algorithms, since no model of the MDP is known, the agent must explore the MDP to obtain information. This introduces a \textbf{exploration-exploitation trade-off} that must be balanced to obtain an optimal policy.\newline
We can also talk about model-based reinforcement learning algorithms in which the agent does not possess, at the beginning, a model of the environment, but estimates
it after a certain amount of time. Once a reasonable model of the environment has been induced, the agent can then apply dynamic programming algorithms to compute the optimal policy.
\subsection{Dynamic Programming based algorithms or Model-based algorithms}
Taking a complex problem, breaking it down into simpler components and solving them recursively is the main idea of dynamic programming.\newline
There are two main Dynamic Programming algorithms from which we can derive others. Policy iteration \cite{howard1960dynamic} and value iteration \cite{bellman1957dynamic}. These two algorithms are based on two principal concepts (\textit{policy evaluation} and \textit{policy improvement}) below explained :
\begin{enumerate}
	\item Policy evaluation :\newline
		The partial order defined by the Relation \ref{PartialOrdering} gives a way of comparing two policies. Now, by establishing an order between value functions, we can find the best policy. But before doing that, we have to evaluate each specific policy.\newline
		Policy evaluation consists on finding the value function using the function $v_\pi$ defined in Relation \ref{equation : Bellman-for-state-consistency} iteratively until convergence. This is done using the following discrete version of the value function :
		\begin{equation}
			v_{\pi}^{k+1}(s) \longleftarrow \sum_{a}\pi(a|s)\cdot\left(r(s,a)+ \gamma\cdot\sum_{s'}p(s'|s,a)\cdot v^{k}_\pi(s')\right)
		\end{equation}
	The algorithm \ref{alg:policy_evaluation} below describes how the values of each state are computed iteratively to evaluate the policy \cite{TheArtOfRL} :
	\begin{algorithm}[H]
		\caption{Policy Evaluation Algorithm}
		\label{alg:policy_evaluation}
		\begin{algorithmic}[1]
			\State Policy ($\pi$), Discount rate ($\gamma$), Threshold to stop ($\varepsilon$)
			\Comment{Input}
			\State $v_\pi(s)= 0$ for all $s \in \mathcal{S}$
			\Comment{Initialization}
			\State The estimated state-value function for the policy $\pi$ (all $v(s)$)
			\Comment{Output}
			\While{not converge}
			\State $\delta = 0$
			\For{each state $s \in S$}
			\State $v \gets v_{\pi}(s)$
			\State $v_{\pi}(s) \gets \sum_{a}\pi(a|s)\cdot\left[r(s,a)+ \gamma\cdot\sum_{s'}p(s'|s,a)\cdot v_\pi(s')\right]$
			\State $\delta \gets \max(\delta, |v - v_\pi(s)|)$
			\EndFor
			\If{$\delta < \varepsilon$}
			\State \textbf{break}
			\EndIf
			\EndWhile
		\end{algorithmic}
	\end{algorithm}
%	\begin{center}
%		\captionof{table}{Policy Evaluation (Iterative Algorithm)}
%		\label{tab:policy_evaluation}
%		\begin{tabular}{lll}
%			\hline
%			\hline
%			\textbf{Input:}  & Policy to be evaluated ($\pi$), Discount rate ($\gamma$), Threshold to stop ($\varepsilon$)  & \\
%			\hline
%			\textbf{Initialization:} & $v_\pi(s)= 0$ & for all $s \in \mathcal{S}$ \\
%			\hline
%			\textbf{Output:}  & The estimated state-value function for the input policy $\pi$ (all $v(s)$) & \\
%			\hline
%			\hline
%			1  & \textbf{while} not converge  do  & \\
%			3  & ~~~$\delta = 0$ & \\
%			4  & ~~~\textbf{for} $s \in S$ \textbf{do}  & \\
%			5  & ~~~~~~$v \longleftarrow v_{\pi}(s)$  &  \\
%			6  & ~~~~~~$v_{\pi}(s) \longleftarrow \sum_{a}\pi(a|s)\cdot\left[r(s,a)+ \gamma\cdot\sum_{s'}p(s'|s,a)\cdot v_\pi(s')\right]$  & \\
%			7  & ~~~~~~$ \delta = \max(\delta, |v - v_\pi(s)|)$ &  \\
%			8  & ~~~\textbf{end for} & \\
%			9  & ~~~\textbf{if} $\delta < \varepsilon$ \textbf{then} & \\
%			10 & ~~~~~~ \textbf{break} & \\
%			11 & ~~~\textbf{end if} & \\
%			\hline
%		\end{tabular}
%	\end{center}
	
	\item Policy improvement :\newline
	In RL, the final goal is to find the optimal policy. Since the policy space is very huge, the task becomes infeasible. This is why, the policy improvement algorithm, as a solution to that problem, aim to find iteratively the optimal policy, by making small improvements on the current one.\newline
	The idea is to modify the current policy greedily with respect to the value function. But, here we should definitely use the action-value function instead of the state one because the policy is about choosing an action. This will result in a deterministic policy (i.e. the probabilities for each case are $0$ or $1$). The only remaining problem is to compute at each time the action-value function but we don't have an estimate of it. Fortunately, as we have already a way to get iteratively the state-value function (using \textbf{policy evaluation}), we can now use it in the expression of the action-value function.\newline
	Knowing the relation between $q$ and $v$, (Relation \ref{equation : relation-action-and-state-value-function}), we can modify the Expression \ref{eq:QvalueGen} as follows :
	\begin{equation}\label{eq:QvalueBetter}
		q_\pi(s,a) = r(s,a) + \gamma\cdot\sum_{s'}p(s'|s,a)\cdot v_\pi(s')
	\end{equation}
	So, the algorithm \ref{alg:policy_improvement} below gives with more details how to improve gradually the policy \cite{TheArtOfRL} :
	
	\begin{algorithm}[H]
		\caption{Policy Improvement Algorithm}
		\label{alg:policy_improvement}
		\begin{algorithmic}[1]
			\State \textbf{Input:} Estimated value function $v_\pi$ , discount rate ($\gamma$)
			\State \textbf{Initialization:} $q_\pi(s,a)= 0$ for all $s \in \mathcal{S}$ and $a\in \mathcal{A}$
			\State \textbf{Output:} Improved deterministic policy $\pi'$
			
			\Comment{Calculation of the state-action value function using the estimated state value function}
			\For{each state $s \in \mathcal{S}$}
			\For{each action $a \in \mathcal{A}(s)$}
			\State $q_\pi(s,a) \gets r(s,a) + \gamma\cdot\sum_{s'}p(s'|s,a)\cdot v_\pi(s')$
			\EndFor
			\EndFor
			
			\Comment{Calculation of an improved deterministic policy}
			\For{each state $s \in \mathcal{S}$}
			\State $a_{*} = \underset{a}{argmax}~ q_\pi(s,a)$
			\For{each action $a \in \mathcal{A}(s)$}
			\If{$a = a_{*}$}
			\State $\pi'(a|s) \gets 1$
			\Else
			\State $\pi'(a|s) \gets 0$
			\EndIf
			\EndFor
			\EndFor
		\end{algorithmic}
	\end{algorithm}
	
%	\begin{center}
%		\captionof{table}{Policy Improvement (Iterative Algorithm)}
%		\label{tab:policy_improvement}
%		\begin{tabular}{lll}
%			\hline
%			\hline
%			\textbf{Input:}  & Estimated value function  $v_\pi$ and the discount rate ($\gamma$)  & \\
%			\hline
%			\textbf{Initialization:} & $q_\pi(s,a)= 0$ & for all $s \in \mathcal{S}$ and $a\in \mathcal{A}$ \\
%			\hline
%			\textbf{Output:}  & Improved deterministic policy $\pi'$ & \\
%			\hline
%			\hline
%			1 & /*** Compute state-action value function using  & \\
%			& ~~~~~~~~~~~estimated state value
%function***/ & \\
%			2  & \textbf{for} $s \in \mathcal{S}$ \textbf{do}  & \\
%			3  & ~~~\textbf{for} $a \in \mathcal{A}(s)$ \textbf{do}  &  \\
%			4  & ~~~~~~$q_\pi(s,a) \longleftarrow r(s,a) + \gamma\cdot\sum_{s'}p(s'|s,a)\cdot v_\pi(s')$  & \\
%			& & \\
%			5  & /*** Compute an improved deterministic policy ***/ & \\
%			6  & \textbf{for} $s \in \mathcal{S}$ \textbf{do}  & \\
%			7  & ~~~$a_{*} = \underset{a}{arg max}~ Q_\pi(s,a) $ & \\
%			8  & ~~~\textbf{for} $a \in \mathcal{A}(s)$ \textbf{do}  & \\
%			9  & ~~~~~~\textbf{if} $a =a_{*} $ \textbf{then} &  \\
%			10  & ~~~~~~~~~$\pi'(a|s) = 1 $ & \\
%			11  & ~~~~~~\textbf{else} & \\
%			12  & ~~~~~~~~~$\pi'(a|s) = 0 $ & \\
%			\hline
%		\end{tabular}
%	\end{center}
Now we can look easily on the value iteration and the policy iteration, which are our goals for these Dynamic Programming techniques of solving Markov Decision Processes.
\end{enumerate}
\begin{enumerate}
	\item \textbf{Policy iteration} :\newline
	Policy iteration is used to find the optimal policy in a Markov decision process (MDP). This algorithm consists of two main steps, policy evaluation and policy improvement. Those two steps are iteratively applied to refine the policy until it converges to the optimal one.
	\begin{center}
		\includegraphics[scale=0.48]{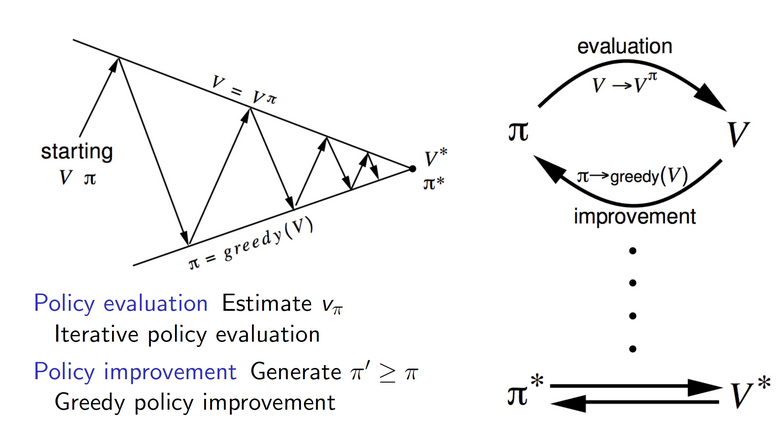}
		\captionof{figure}[Policy iteration]{Policy iteration \cite{silver2015lecture}}\label{fig:policyIteration}
	\end{center}
	\item \textbf{Value iteration} :\newline
	This value iteration algorithm is like an improvement of policy iteration. In fact, the latter algorithm works effectively and find the optimal policy but it can take a lot of time because at each epoch it computes the policy evaluation and the policy improvement. Value iteration combines policy evaluation and policy improvement into one step. The idea is to replace the general expression of the value function for a given policy (Expression \ref{equation : Bellman-for-state-consistency}) by the one giving the optimal value function (Expression \ref{equation : Bell-Optim-state-val-func}) in the Expression \ref{eq:QvalueBetter}.
	\begin{center}
		\includegraphics[scale=0.7]{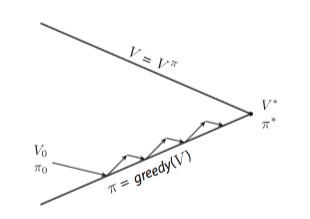}
		\captionof{figure}[Value Iteration]{Value Iteration \cite{stackoverflowValueIter}}\label{fig:ValueIteration}
	\end{center}
\end{enumerate}
Generally, DP algorithms provide a powerful framework for solving reinforcement learning problems when a perfect model of the environment is known. It decomposes the problem into simpler subproblems and then uses the Bellman equations to compute the value functions. Optimality is reached when the expected cumulative reward is maximised. \newline
Unfortunately, the model of the environment is not always available. Also, DP algorithms can be computationally expensive, especially for problems with large state and action spaces. For those reasons, other RL methods have been developed, such as Monte Carlo methods and temporal difference learning \cite{sutton1995generalization}.

\subsection{Model-free algorithms : Monte Carlo}
Dynamic Programming (DP) assumes that the agent has a perfect access to the model of the environment (dynamics of the environment). But most of the time, we don't have access to that. Additionally, sometimes the state and/or action spaces are very huge in such a way that DP algorithm can be inefficient.\newline
Monte Carlo refers to the method of using experience to estimate value functions by averaging sample returns. Statistically it is proven to converge to the true estimate of the returns \cite{bertsekas2008introduction}. It uses sequences of state-action-reward samples generated by the agent by interacting with the environment.\newline
There are always two important steps to be taken into account, the evaluation step and the improvement step. The idea of Monte Carlo (MC) policy evaluation for example is, following a certain policy $\pi$, to generate a sample of state-action-reward-state or state-action-reward-state-action starting from a certain state. After generating a large amount for many episodes (because here we are talking about processes which have a terminating state, i.e. they can end in a finite number of steps), the expected return is obtained by averaging the returns from a large amount of complete sample episodes. This gives the corresponding value for that state, given that specific policy. As the number increases, the average will converges to the true value of $v_\pi$ \cite{bertsekas2008introduction}.\newline
The previous text just gives a general idea on how MC works. But, there are many implementations of this idea :
\begin{itemize}
	\item First-visit MC policy evaluation (using value function or action-value function)
	\item Every-visit MC policy evaluation (using value function or action-value function)
	\item Incremental MC policy evaluation (using value function or action-value function)
	\item Monte Carlo Control (Monte Carlo Policy Improvement).
\end{itemize}
Here is a succinct presentation of those main MC approaches :
\begin{enumerate}
	\item \textbf{First-visit Monte Carlo Policy Evaluation}\newline
	The idea of this method is to generate a lots of samples, compute the returns at each step within each episode and update the estimate of the state value function for each state (at the end of the considered episode). The update is done by observing whether it is the first time to visit the state in each episode and then update the return and the number of visits. At the end, the expected return will be the average of the returns. What we have just explained is summarised in the following algorithm \cite{TheArtOfRL} :
	
	\begin{algorithm}[H]
		\caption{First-visit Monte Carlo Policy Evaluation for $v_\pi$}
		\label{alg:MC_first_visit_evaluation}
		\begin{algorithmic}[1]
			\State Policy to evaluate ($\pi$), Discount rate ($\gamma$), Number of episodes ($K$)
			\Comment{Input}
			\State $i=0, ~ v_\pi(s)= 0, ~ G(s)=0, ~ N(s)=0$ for all $s \in \mathcal{S}$
			\Comment{Initialization}
			\State The estimated state-value function for $\pi$ (all $v(s)$)
			\Comment{Output}
			\While{$i< K$}
			\State Generate a sample episode $\tau$ by following the policy $\pi$, where $\tau= S_0, R_1, S_1, R_2,...$
			\State Compute the return $G_t$ for every time step $t= 1, 2, 3, 4,...$ in episode $\tau$
			\For{each time step $t$ till the end of episode $\tau$}
			\If{it's the first time state $S_t$ is visited in episode $\tau$}
			\State $N(S_t) \leftarrow N(S_t) + 1$
			\Comment{Increment counter of total visits}
			\State $G(S_t) \leftarrow G(S_t) + G_t$
			\Comment{Increment total return}
			\State $v_\pi(S_t) \leftarrow G(S_t)/N(S_t)$\label{UpdateInAlgo}
			\Comment{Update estimated value}
			\EndIf
			\EndFor
			\State $i = i+1$
			\EndWhile
		\end{algorithmic}
	\end{algorithm}
	
%	\begin{center}
%		\captionof{table}{First-visit Monte Carlo Policy Evaluation for $v_\pi$}
%		\label{tab:MC_first_visit_evaluation}
%		\begin{tabular}{lll}
%			\hline
%			\hline
%			\textbf{Input:}  & Policy to be evaluated ($\pi$), Discount rate ($\gamma$), Number of episodes ($K$)  & \\
%			\hline
%			\textbf{Initialization:} & $i=0, ~ v_\pi(s)= 0, ~ G(s)=0, ~ N(s)=0$  & for all $s \in \mathcal{S}$ \\
%			\hline
%			\textbf{Output:}  & The estimated state-value function for the input policy $\pi$ (all $v(s)$) & \\
%			\hline
%			\hline
%			1  & \textbf{while} $i< K$ \textbf{do}  & \\
%			2  & ~~~Generate a sample episode $\tau$ by following policy $\pi$,& \\ 
%			& ~~~~~~~~~~~~~~~~~~~~~~~~~~~~~~~~~~~~~~~~~~~~~~~~where $\tau= S_0, R_0, S_1, R_1,...$ & \\
%			3  & ~~~Compute the return Gt for every time step $t= 1, 2, 3, 4,...$ in episode $\tau$ & \\
%			4  & ~~~\textbf{for} each time step $t$ till the end of episode $\tau$ \textbf{do}  & \\
%			5  & ~~~~~~\textbf{if} it's the first time state $S_t$ is visited in episode $\tau$ \textbf{then}  &  \\
%			6  & ~~~~~~~~~Increment counter of total visits: $N(S_t) \longleftarrow N(S_t) + 1$  & \\
%			7  & ~~~~~~~~~Increment total return $G(S_t) \longleftarrow G(S_t) + G_t$ &  \\
%			8  & ~~~~~~~~~Update estimated value $v_\pi(S_t) \longleftarrow G(S_t)/N(S_t)$ & \\
%			9  & ~~~$i = i+1$ & \\
%			\hline
%		\end{tabular}
%	\end{center}	
	\item \textbf{Every-visit Monte Carlo Policy Evaluation}\newline
	A visit to a state $s$ is each time the agent passes to that specific state. In the first-visit algorithm, only the first time is considered in each episode while updating the return. For the every-visit, all the visits, even within an episode are taken into account. The former method is unbiased but has an higher variance, the latter is biased with a lower variance. The trade-off is now taken by considering that aspect. Sometime even a combination of them is envisioned. The algorithm comes from a small modification on the one for \textit{first-visit} \cite{TheArtOfRL, sutton2018reinforcement}.
	\item \textbf{Incremental Monte Carlo Policy Evaluation}\newline
	The update method used line \ref{UpdateInAlgo} of the previous algorithm \ref{alg:MC_first_visit_evaluation} is still perfectible. In fact, many researches have shown that the incremental updates converges better than the simple update as done in the previous algorithm \cite{TheArtOfRL}. So, instead of $v_\pi(S_t) \longleftarrow G(S_t)/N(S_t)$, the incremental approach consists on doing :
	\begin{equation}\label{eq:incrementalUpdate}
		v_\pi(S_t) \longleftarrow v_\pi(S_t)+\dfrac{1}{N(S_t)}\cdot \underset{\text{incremental error}}{\underbrace{\left(G(S_t)-v_\pi(S_t)\right)}}
	\end{equation}
\end{enumerate}
All these algorithms, can also be used by considering the action-value function instead of the state value function. We just have to change mutatus mutandi everything. For example, for the incremental update shown in Relation \ref{eq:incrementalUpdate},  we simply do :
\begin{equation}\label{eq:incrementalUpdateQ}
	q_\pi(S_t,A_t) \longleftarrow q_\pi(S_t,A_t)+\dfrac{1}{N(S_t,A_t)}\cdot \Big(G_t-q_\pi(S_t,A_t)\Big)
\end{equation}
Also, exactly as for Dynamic Programming, the control step can be considered. The control step consists now on changing greedily the policy. For this case, there is for example \textbf{Monte Carlo Control or Policy Improvement}. For this, we can even try to do not always choose the best value but combine exploration and exploitation in a good way for better results \cite{van2012reinforcement}.
\subsection{Model-free algorithms : Temporal Difference}
MC is limited to episodic problems. Temporal Difference (TD) methods generalize MC by dealing with both episodic and continuing (no natural ending or terminating state for the considered task) or infinite RL problems.\newline
So, MC has two main drawbacks : it only takes into account episodic problems and we have to wait for the end of an episode to do any update on the estimated values. This is called Temporal Difference because the agent update its estimates based on the immediate reward and the estimated value for the next state without waiting for the end, even if these values will change another time during the episode. The main idea is a continuous updating of values, a guess from guesses.\newline Let's give the most important implementations which really show how TD works :
\begin{enumerate}
	\item \textbf{TD Policy evaluation :}\newline
	For a given policy, samples are generated and at the same time incremental update is applied in the following way, without waiting for the end :
	\begin{align}\label{eq:UpdateTDState}
		v_\pi(S_t) &\longleftarrow v_\pi(S_t) + \lambda\cdot\left(G_t-v_\pi(S_t)\right)\nonumber\\
		&\longleftarrow v_\pi(S_t) + \lambda\cdot\Big(\left[R_t+\gamma\cdot v_\pi(S_{t+1})\right]-v_\pi(S_t)\Big)
	\end{align}
	We here show the corresponding algorithm :
	\begin{algorithm}[H]
		\caption{Temporal Difference TD(0) (policy evaluation for $v_\pi$)}
		\label{alg:TD_evaluation}
		\begin{algorithmic}[1]
			\State \textbf{Input:} Policy ($\pi$), Discount rate ($\gamma$), Step size $\lambda$, Number of episodes ($K$)
			\State Sample initial state $S_0$. Initialization of $i=0$, $v_\pi(s)=0$ for all $s \in \mathcal{S}$.
			\Comment{Initialization}
			\While{$i<K$}
			\State Sample action $A_t$ for $S_t$ following the policy $\pi$
			\State Chose the action $A_t$ in the environment and get $R_t$ while observing $S_{t+1}$
			\State $i=i+1$
			\State Compute TD target:
			\State \quad $\delta_t = \left\{\begin{array}{ll}
				R_t & \text{if the next state $S_{t+1}$ is a terminal state} \\
				R_t+\gamma\cdot v_\pi(S_{t+1}) & \text{otherwise, we are bootstrapping as shown here}
			\end{array}\right.$
			\State Update estimated value $v_\pi(S_t) \leftarrow v_\pi(S_t) + \left(\delta_t-v_\pi(S_t)\right)$
			\State $S_t = S_{t+1}$
			\EndWhile
		\end{algorithmic}
	\end{algorithm}
	
%	\begin{center}
%		\captionof{table}{Temporal difference [TD(0)] policy evaluation for $v_\pi$}
%		\label{tab:TD_evaluation}
%		\begin{tabular}{lll}
%			\hline
%			\hline
%			\textbf{Input:}  & Policy ($\pi$), Discount rate ($\gamma$),  step size $\lambda$, Number of episodes ($K$)  & \\
%			\hline
%			\textbf{Initialization:} & $i=0, ~ v_\pi(s)= 0$  & for all $s \in \mathcal{S}$ \\
%			\hline
%			\textbf{Output:}  & The estimated state-value function for the input policy $\pi$ (all $v(s)$) & \\
%			\hline
%			\hline
%			1  & Sample initial state $S_0$  & \\
%			1  & \textbf{while} $i< K$ \textbf{do}  & \\
%			2  & ~~~Sample action $A_t$ for $S_t$ following $\pi$ & \\
%			3  & ~~~Take action $A_t$ in environment and observe $R_t$ and $S_{t+1}$ & \\
%			3  & ~~~$i=i+1$ & \\
%			4  & ~~~Compute TD target :  & \\
%			5  & ~~~~~~$\delta_t = \left\{\begin{array}{ll}
%				R_t & \text{if the next state $S_{t+1}$ is a terminal state} \\
%				R_t+\gamma\cdot v_\pi(S_{t+1}) & \text{otherwise, we are boostraping as shown here}
%			\end{array}\right.$  &  \\
%			6  & ~~~~~~Update estimated value $v_\pi(S_t)\longleftarrow v_\pi(S_t) + \left(\delta_t-v_\pi(S_t)\right) $  & \\
%			7  & ~~~~~~$S_t = S_{t+1}$ & \\
%			\hline
%		\end{tabular}
%	\end{center}
	If we need to work with action-value functions, we just have to replace $v_\pi$ by $q_\pi$ coherently in equation \ref{eq:UpdateTDState}. One obtains :
	\begin{align}\label{eq:UpdateTDQFunc}
		q_\pi(S_t,A_t) &\longleftarrow q_\pi(S_t,A_t) + \lambda\cdot\left(G_t-q_\pi(S_t,A_t)\right)\nonumber\\
		&\longleftarrow q_\pi(S_t,A_t) + \lambda\cdot\Big(\left[R_t+\gamma\cdot q_\pi(S_{t+1},A_{t+1})\right]-q_\pi(S_t,A_t)\Big)
	\end{align}
	At the end of the algorithm, we'll update the state as well as the action in case we are using the Expression \ref{eq:UpdateTDQFunc}.\newline
	Based on this idea, the update can go far than one step as in the Expression \ref{eq:UpdateTDState}.
	In fact, we can replace $G_t = R_{t}+\gamma\cdot v_\pi(S_{t+1})$ by $G_t = R_{t}+\gamma\cdot R_{t+1}+ \cdots + \gamma^n\cdot v_\pi(S_{t+n})$. Depending on when we stop, it gives us different levels of bootstrapping. This generalizes the TD(0) algorithm we have been illustrating in this part. So, in general we have just an \textbf{n-step TD Policy Evaluation} that we choose to write \textbf{TD(n)} with \textbf{n+1} the number of steps ahead in the formula \cite{TheArtOfRL,sutton2018reinforcement}.
	\item \textbf{TD Control : SARSA}\newline
	SARSA stands literally for State-Action-Reward-State-Action. This method uses in fact the general policy iteration template that we have shown in the Dynamic Programming part to find the optimal policy. But it uses specific methods to manage exploration and exploitation while searching for the optimal policy. In fact, from the previous policy iteration that we saw in the DP part, we have to take at each time the one which gives the maximum return (\textbf{exploitation}). But for this case, we can sometimes take a suboptimal action (\textbf{exploration}) with a certain probability $\varepsilon$. This is what we call an \textbf{$\varepsilon\text{-greedy}$ selection method}.\newline
	Simply, the main idea is to loop on all the states for all the actions (state-action pairs) and then use an $\varepsilon\text{-greedy}$ policy to select the best action. But note that there are many ways of managing the exploration-exploitation trade-off \cite{TheArtOfRL}.\newline
	In general also, we can use \textbf{n-step TD Policy Evaluation} in the evaluation part of this policy iteration. It now gives raise to all other implementations of SARSA \textbf{n-step SARSA} \cite{sutton2018reinforcement}.
	\item \textbf{TD Control : Q-Learning}\newline
	The SARSA algorithm previously introduced is in fact a family of general policy iteration
algorithm which estimates the state-action value function $q_\pi$ for a policy $\pi$ and uses an $\varepsilon\text{-greedy}$ policy to compute a new, better policy $\pi'$. But, as we did for value iteration algorithm, we can skip the step of estimating values for a policy $\pi$ by directly making an estimation using the optimal value function formula, so that we'll be directly estimating the optimal policy. For Q-Learning, we use the same idea by changing the Relation \ref{eq:UpdateTDQFunc} by :
	\begin{equation}\label{eq:UpdateQLearningFunc}
		q(S_t,A_t) \longleftarrow q(S_t,A_t) + \lambda\cdot\left(\left[R_t+\gamma\cdot \underset{a'}{max}~q(S_{t+1}=s',a')\right]-q(S_t,A_t)\right)
	\end{equation}
	One advantage of this algorithm is that it always uses the maximum value $q(S_{t+1},a')$ across all valid
actions for the successor state $S_{t+1}=s'$, regardless of which action is chosen by the policy.\newline
	We should note that, even if Q-Learning is one of the most used in real-world problems, it may be biased. This is why another method called \textbf{Double Q-Learning} was suggested to alleviate that problem \cite{sutton2018reinforcement,van2010double}.
\end{enumerate}
In many of the tasks the state space is very huge. In such cases, even with the most efficient algorithms we have described above, we cannot expect to find an optimal policy or the optimal value function even in the limit of infinite time and data. Most of the time then, the goal is to
find a good approximate of the function we are looking for. Using linear or non-linear approximates. For more information about that, see \cite{TheArtOfRL, ris2023fundamentals, sutton2018reinforcement}.\newline

Now, as we have seen the main approaches to solve an MDP, which in fact resume the key points of RL foundational algorithms, we can now discuss on how it is related to the fixed point theory.
\section*{Chapter Conclusion}
In this chapter, we have shown the fundamentals of Reinforcement Learning and how it is connected to Markov Decision Processes. After that, we have discussed about the optimal value functions and optimal policies and then at the end we have shown how generally the optimality is accomplished in RL. The next chapter will now formalize what we have just presented in this chapter as formula to find the optimal policy in terms of operators. This will now allow us to show the connection with the Banach Contraction Principle and how it can help understand deeply why and how RL algorithms work.
\chapter{Bellman Operators and Convergence of RL algorithms}\label{chap3}
In this chapter, we will express reinforcement in terms of operators so that we have a simple way to prove why reinforcement learning algorithms work well and find the optimal policy. To do this, we will first write the Banach contraction principle on normed spaces in a way that makes it easy to use in our next proofs, then we will discuss the Bellman operators, and finally we will talk about the core methods for finding optimality introduced in the previous chapter (but now using the language introduced in this chapter).
\section{Refinement of the Banach Contraction Principle}
We need to state the Banach Contraction Principle presented in the first chapter in such a way that everything in this part will be expressed in the same way, using the same language.
\subsection{Contraction mapping and fixed-point}
Let us now refine everything in term of norms in such a way that we will have a consistent notation until the end of this work. In fact, as we can see in this work, all the concepts related to reinforcement learning are defined in term of norms. So,
\begin{defn}
	Let $(X,||\cdot||)$ be a normed space. A mapping $\mathcal{T} : X \rightarrow X$ is a $\gamma-$contraction mapping if, there exists a $\gamma\in [0,1)$ such that, for any $x_1,x_2\in X$ we have 
	\[||\mathcal{T}x_1-\mathcal{T}x_2||\leq \gamma\cdot ||x_1-x_2||\]
	From this, each contraction is Lipchitz and thus continuous \cite{kumaresan2005topology}. It implies that :
	\[\text{If} ~~ x_n\underset{||\cdot||}{\longrightarrow}x~~~\text{then}~~~ \mathcal{T}x_n\underset{||\cdot||}{\longrightarrow}\mathcal{T}x,\]
	where $\underset{||\cdot||}{\longrightarrow}$ denotes the convergence using that norm $||\cdot||$.
	Now, by definition, an element $x^*\in X$ is called fixed point for $\mathcal{T}$ if we have :
	\[\mathcal{T}x^* = x^*\]
\end{defn}

\begin{pro}\label{thm:RefinedBCP}
	Let $(X,||\cdot||)$ be a Banach space and~~ $\mathcal{T} : X\rightarrow X$ ~~ a ~~ $\gamma-\text{contraction mapping}$. Then :
	\begin{enumerate}
		\item $\mathcal{T}$ has a unique fixed point $x^*\in X$.
		\item For every starting point $x_0 \in X$, the sequence recursively defined by $x_{n+1}=\mathcal{T}x_n$ converges to $x^*$ in a geometric fashion.
	\end{enumerate}
\end{pro}
\begin{proof}
    For the  first part, see the proof of the Banach fixed point Theorem \ref{thm:BCP}. Let's now only justify the fact that the convergence is in geometric fashion :
    \[\mathcal{T}||x_{n-1}-x^*||=||\mathcal{T}x_{n-1}-\mathcal{T}x^*||=||\mathcal{T}x_{n-1}-x^*||= ||x_{n}-x^*||,\]
    and as we are working with a contraction mapping, we also have $||\mathcal{T}x_{n-1}-\mathcal{T}x^*|| \leq \gamma \cdot ||x_{n-1}-x^*||$, then we can conclude directly $||x_{n}-x^*|| \leq \gamma \cdot ||x_{n-1}-x^*||$. Which implies that we get, by iteratively applying the mapping on the second member of this previous inequality, $||x_n-x^*||\leq \gamma^n\cdot ||x_0-x^*||$ thus 
    \[\lim_{n\to\infty} ||x_n-x^*||\leq\lim_{n\to\infty}\left(\gamma^n\cdot ||x_0-x^*||\right)=0\Rightarrow \lim_{n\to\infty} ||x_n-x^*|| = 0\]
\end{proof}
As we may have noticed, the Proposition \ref{thm:RefinedBCP}, is the Banach Contraction Principle as given by the Theorem \ref{thm:BCP}. We restrain it on Banach spaces because this is sufficient for the RL setting. In fact, in many instances encountered in Reinforcement Learning (RL), the vector space $X$ is frequently identified as $\mathbb{R}^d$ or the space of real bounded functions, and the most relevant norms are mostly $||\cdot||_{2}$ and $||\cdot||_{\infty}$. So, stating the theorem on Banach Spaces is sufficient in this setting.
\section{Bellman Optimality Operators}
Operators are maps in space of functions. Here we will be defining them with respect to the associated RL equations. Let's first define the whole mathematical framework that we'll be using here.\newline
\phantom{c}\newline
Let $\mathcal{M}=\langle \mathcal{S}, \mathcal{A},p,r, \gamma \rangle$ be a Markov Decision Process (MDP), $\mathcal{V}_s\equiv\mathcal{V}$ the space of bounded real-valued functions over $\mathcal{S}$ and $\mathcal{Q}$ the space of bounded real-valued functions over $\mathcal{S}\times\mathcal{A}$. With that we will call :
\begin{itemize}
	\item $ \mathcal{V} $ : ~~ the space of all state-value functions,
	\item $\mathcal{Q}$ : ~~ the space of all action-value functions,
	\item $\left(\mathcal{T}^\pi_v\right) : \mathcal{V}\rightarrow\mathcal{V}$ ~~ the Bellman Expectation Operator for the state value function,
	\item $\left(\mathcal{T}^{\pi}_Q\right) : Q\rightarrow Q$ : ~~ the Bellman Expectation Operator for the action-value function,
	\item $\left(\mathcal{T}^*_v\right) : \mathcal{V}\rightarrow\mathcal{V}$ ~~ the Bellman Optimality Operator for the state value function,
	\item $\left(\mathcal{T}^{*}_Q\right) : Q\rightarrow Q$ : ~~ the Bellman Optimality for the action-value function
	\item $\pi : \mathcal{S} \rightarrow \Delta\left(\mathcal{A}\right)$ ~~ the policy (a conditional probability defined by $\pi(a|s) = Pr\{a|s\}$) with $\Delta\left(\mathcal{A}\right)$ denoting a probability distribution of actions.
\end{itemize}
Inspired by the definition given in the Expression \ref{MetricBoundedFunc} and the Example \ref{Exa:Banach} we know that, with $||f(s)||_\infty \equiv \underset{s}{\max}|f(s)|$, the tuple $(\mathcal{V}, ||.||_\infty)$ is a Banach Space.\newline
Also, with $||f(s,a)||_\infty \equiv \underset{s,a}{\max}|f(s,a)|$ we can also see that $(\mathcal{Q}, ||.||_\infty)$ is a Banach Space. Those two supremum norms are different but we will be writing them in the same way because, the real formula will depend on the space we are working with.
\begin{defn}
	Let's recall the fact that, the Bellman Optimality equation for the value function (Relation \ref{equation : Bell-Optim-state-val-func}) is given by the following expression :
	\[v_*(s) = \underset{a}{\max}\Big(r(s,a) + \gamma\cdot\sum_{s'}p(s'|s,a)\cdot v_{*}(s')\Big)\]
	By analogy, we will define, point-wise, the \textbf{Bellman Optimality Operator for state value function}  $\left(\mathcal{T}^*_v\right) : \mathcal{V}\rightarrow\mathcal{V}$ by :
	\begin{equation}\label{equation : Bellman-Operator}
		\left(\mathcal{T}^*_v f\right)(s) \equiv \underset{a}{\max}\left[r(s,a)+\gamma\cdot\sum_{s'}p(s'|s,a)\cdot f(s')\right]~~~\forall f\in \mathcal{V}
	\end{equation}
\end{defn}

\subsection{Properties of this operator}
The convention we'll take here is to write $\left(\mathcal{T}^*_v \right)$ just $\left(\mathcal{T}^* \right)$, knowing that the space will be written $\mathcal{V}$ by default. Now, what is very interesting is that, the operator, defined in Relation \ref{equation : Bellman-Operator} has these properties \cite{deepminducl2021} :
\begin{itemize}
	\item $\mathcal{T}^*$ is a $\gamma-$contraction with respect to the supremum norm $||\cdot||_{\infty}$ on $\mathcal{V}$ :
	\[||\mathcal{T}^*u-\mathcal{T}^*v||_\infty\leq \gamma\cdot ||u-v||_\infty~~\forall~ u,v\in\mathcal{V}~~~\text{with $\gamma\in [0,1)$}\]
	\item $\mathcal{T}^*$ is monotonic :
	\[\forall~ u,v \in \mathcal{V}~\text{s.t}~u\leq v,~\text{in any state,} ~~\mathcal{T}^*u\leq \mathcal{T}^*v \]
\end{itemize}
\begin{proof}
	First, let us rewrite simply the Bellman Optimality Operator in Relation \ref{equation : Bellman-Operator} like this :
	\begin{equation}\label{equation : Bellman-Operator-written-with-Expectation}
		\left(\mathcal{T}^*_v f\right)(s) \equiv \left(\mathcal{T}^* f\right)(s) \equiv \underset{a}{\max}\Big[r(s,a)+\gamma\cdot\mathbb{E}_{s'|s,a}f(s')\Big]
	\end{equation}
	Now, let us prove the two properties of the Bellman Optimality Operator $\mathcal{T}^*$ :
	\begin{enumerate}
		\item Using the definition of a contraction mapping and Bellman Optimality Operator we expand :
		\[\left| \mathcal{T}^*u(s)-\mathcal{T}^*v(s) \right| = \left| \underset{a}{\max}\left[r(s,a)+\gamma\cdot\mathbb{E}_{s'|s,a}u(s')\right]-\underset{b}{\max}\left[r(s,b)+\gamma\cdot\mathbb{E}_{s'|s,b}v(s')\right] \right|\]
		Since for $f,g$ real-valued functions, $|\underset{x}{\max}f(x)-\underset{x}{\max}g(x)|\leq\underset{x}{\max}|f(x)-g(x)|$, it follows that :
		\begin{align}
			\left|\mathcal{T}^*u(s)-\mathcal{T}^*v(s)\right| &\leq \underset{a}{\max}\big| \left[r(s,a)+\gamma\cdot\mathbb{E}_{s'|s,a}u(s')\right]-\left[r(s,a)+\gamma\cdot\mathbb{E}_{s'|s,a}v(s')\right] \big|\nonumber\\
			&= \underset{a}{\max}\big| \gamma\cdot\mathbb{E}_{s'|s,a}u(s')-\gamma\cdot\mathbb{E}_{s'|s,a}v(s') \big|\nonumber\\
			&= \underset{a}{\max}\left(\gamma\cdot\big| \mathbb{E}_{s'|s,a}\left(u(s')-v(s')\right) \big|\right)\nonumber\\
			&\leq \underset{s'}{\max}\big| \gamma\cdot\left(u(s')-v(s')\right) \big|\nonumber\\
			&= \gamma\cdot\underset{s}{\max}\big| \left(u(s)-v(s)\right) \big|\nonumber\\
			\Rightarrow\left|\mathcal{T}^*u(s)-\mathcal{T}^*v(s)\right| &\leq \gamma\cdot ||u-v||_\infty\nonumber\\
			\Rightarrow\underset{s}{\max}\left|\mathcal{T}^*u(s)-\mathcal{T}^*v(s)\right| &\leq \gamma\cdot ||u-v||_\infty\nonumber\\
			\Rightarrow||\mathcal{T}^*u(s)-\mathcal{T}^*v(s)||_\infty &\leq \gamma\cdot ||u-v||_\infty
		\end{align}
		\item We can now prove the monotonicity very simply :\newline
		\newline
		Assume that for all $s$, $v(s)\leq u(s)$. Then $ r(s,a)+\gamma\cdot\mathbb{E}_{s'|s,a}v(s')\leq r(s,a)+\gamma\cdot\mathbb{E}_{s'|s,a}u(s')$.\newline
		\newline
		Let us evaluate $\mathcal{T}^*v(s)-\mathcal{T}^*u(s)$ and see what happens :
		\begin{align}
			\mathcal{T}^*v(s)-\mathcal{T}^*u(s) ~ &=\underset{a}{\max}\left[r(s,a)+\gamma\cdot\mathbb{E}_{s'|s,a}v(s')\right]-\underset{a}{\max}\left[r(s,a)+\gamma\cdot\mathbb{E}_{s'|s,a}u(s')\right]\nonumber\\
			&\leq \underset{a}{\max}\Big( \underset{\leq~~ 0}{\underbrace{\left[r(s,a)+\gamma\cdot\mathbb{E}_{s'|s,a}v(s')\right] - \left[r(s,a)+\gamma\cdot\mathbb{E}_{s'|s,a}u(s')\right]}}\Big)\nonumber\\
			\mathcal{T}^*v(s)-\mathcal{T}^*u(s) ~ &\leq ~~0\nonumber\\
			&\nonumber\\
			\Rightarrow \mathcal{T}^*v(s) ~~ &\leq ~~ \mathcal{T}^*u(s)~~~\forall ~ s
		\end{align}
	\end{enumerate}
\end{proof}
Now, using the \textit{Banach Contraction Principle} as refined by the Proposition \ref{thm:RefinedBCP}, we can conclude that, the Bellman Optimality Operator $\mathcal{T}^*$ on $(\mathcal{V},||\cdot||_{\infty})$ has the following behaviors :
\begin{itemize}
	\item $\mathcal{T}^*$ has a unique fixed point $f^*\in \mathcal{V}$.
	\item For any starting point (function) $f_0$, we can find the fixed point $f^*$ by looking for the convergence of the sequence defined by $f_{n+1}=\mathcal{T}f_n$.\newline 
	Also, the fixed point that we will find will be the optimal value function, since it'll be the maximum (in fact, this sequence is monotonically increasing by our definition of the problem).
\end{itemize}
This fixed point $f^{*}$ is the overall optimal value function because here the equation use exactly the optimal policy $\pi^*$.\newline
We can, in the same way, define the operator for action-value functions.
\begin{defn}
	As we know, the Bellman Optimality equation for Q-values (Relation \ref{equation : Bell-Optim-action-val-func}) is given by this expression :
	\[q_{*}(s,a) = r(s,a) + \gamma\cdot\sum_{s'}p(s'|s,a)\cdot \underset{a'}{\max}~ {q_{*}(s',a')}\]
	Now, giving the same MDP $\mathcal{M}$, let $\mathcal{Q}$ be the space of bounded real-valued functions over $\mathcal{S}\times\mathcal{A}$.\newline
	We can also define, point-wise, the \textbf{Bellman Optimality Operator for the Q-Value} $\mathcal{T}^{*}_Q : Q\rightarrow Q$ as follows :
	\begin{equation}\label{equation : Bellman-Operator-for-Q}
		\left(\mathcal{T}^{*}_Q f\right)(s,a) \equiv r(s,a) + \gamma\cdot\sum_{s'}\left[p(s'|s,a)\cdot\underset{a'\in\mathcal{A}(s')}{\max}f(s',a')\right]~~~\forall f\in \mathcal{Q}
	\end{equation}
\end{defn}
The same properties can be stated for this Bellman optimality operator for the Q-value, using the same kind of proofs.\newline
The Bellman optimality operators as defined here are particular cases of what we will refer to as \textbf{Bellman Expectation Operators}. In this context, we considered the scenario of selecting the action that maximises the expected reward, without consideration of alternative policies. We now consider the policy in more general terms. This does not lead only to the optimal value function, but also to the value function corresponding to the chosen policy.

\section{Bellman Expectation Operators}
We will start with the state-value Bellman Expectation Operator. We know that, the Bellman Expectation equation is given by the Relation \ref{equation : Bellman-for-state-consistency} :
\[v_{\pi}(s) = \sum_{a}\pi(a|s)\cdot\left(r(s,a)+ \gamma\cdot\sum_{s'}p(s'|s,a) v_\pi(s')\right)\]
We can define by analogy the Bellman Expectation Operator by :
\begin{equation}\label{eq:BellmanExpOperat}
	\left(\mathcal{T}^{\pi}_v f\right)(s) = \sum_{a}\pi(a|s)\cdot\left(r(s,a)+ \gamma\cdot\sum_{s'}p(s'|s,a) f(s')\right) ~~\forall f\in \mathcal{V}
\end{equation}
\begin{pro}
	We claim that, the operator $\mathcal{T}^\pi_v$, as defined by the Expression \ref{eq:BellmanExpOperat} is :
	\begin{enumerate}
		\item A $\gamma-$contraction mapping.
		\item A monotonic mapping.
	\end{enumerate}
\end{pro}
\begin{proof}
	Let consider two value functions $u$ and $v$ in $\mathcal{V}$.
	\begin{enumerate}
		\item Let us compute the absolute value between the transformations of those two functions and see what happens :
		\begin{align}
			\left|\mathcal{T}^{\pi}_v u(s)-\mathcal{T}^{\pi}_v u(s)\right| &= \left| \sum_{a}\pi(a|s)\cdot\left(\gamma\cdot\sum_{s'}p(s'|s,a) u(s')- \gamma\cdot\sum_{s'}p(s'|s,a) v(s')\right) \right|\nonumber\\
			&= \gamma\cdot \left| \sum_{a}\pi(a|s)\cdot\left(\sum_{s'}p(s'|s,a) u(s')- \sum_{s'}p(s'|s,a) v(s')\right) \right|\nonumber\\
			&= \gamma\cdot \left| \sum_{a}\pi(a|s)\cdot\left(\mathbb{E}_{s'|s,a} u(s')- \mathbb{E}_{s'|s,a} v(s')\right) \right|\nonumber\\
			&\leq \gamma\cdot \underset{s'}{\max}\left|u(s')- v(s') \right|\nonumber\\
			\Rightarrow \left|\mathcal{T}^{\pi}_v u(s)-\mathcal{T}^{\pi}_v u(s)\right| &\leq \gamma\cdot \underset{s}{\max}\left|u(s)- v(s) \right|\nonumber\\
			\Rightarrow \underset{s}{\max}\left|\mathcal{T}^{\pi}_v u(s)-\mathcal{T}^{\pi}_v u(s)\right| &\leq \gamma\cdot ||u(s)- v(s) ||_\infty\nonumber\\
			\Rightarrow ||\mathcal{T}^{\pi}_v u(s)-\mathcal{T}^{\pi}_v u(s)||_\infty &\leq \gamma\cdot ||u(s)- v(s) ||_\infty
		\end{align}
		This proves that the Bellman Expectation Operator for the state value function is a contraction mapping.
		\item Let us now prove the monotonicity :\newline
		Assume that $u(s)\leq v(s)$ in all states, then we will get :
		\begin{align}
			\mathcal{T}^{\pi}_v u(s)-\mathcal{T}^{\pi}_v v(s) &= \gamma\cdot \sum_{a}\pi(a|s)\cdot\left(\mathbb{E}_{s'|s,a} u(s')- \mathbb{E}_{s'|s,a} v(s')\right)\nonumber\\
			&= \gamma\cdot \sum_{a}\pi(a|s)\cdot\left(\mathbb{E}_{s'|s,a}\underset{u(s)\leq v(s)~\forall s}{\underbrace{\left( u(s')- v(s')\right)}}\right)\nonumber\\
			&= \gamma\cdot \sum_{a}\pi(a|s)\cdot\left(\mathbb{E}_{s'|s,a}\left( u(s')- v(s')\right)\right)\leq 0\nonumber\\
			\Rightarrow \mathcal{T}^{\pi}_v u(s)-\mathcal{T}^{\pi}_v v(s) &\leq 0\nonumber\\
			\Rightarrow \mathcal{T}^{\pi}_v u(s) &\leq \mathcal{T}^{\pi}_v v(s) ~~~\text{i.e. $\mathcal{T}^{\pi}_v$ is a monotonic mapping}
		\end{align}
	\end{enumerate}
\end{proof}
From the properties we've just proven, using the Banach Contraction Principle as refined in the Proposition \ref{thm:RefinedBCP}, we can conclude that :
\begin{itemize}
	\item $\mathcal{T}^\pi_v$ has a unique fixed point $f^{\pi}\in \mathcal{V}$.
	\item For any starting point $f_0$, we can find the fixed point $f^{\pi}$ by looking for the convergence of the sequence defined by $f_{n+1}^{\pi}=\mathcal{T}f_n^{\pi}$.
\end{itemize}
This fixed point $f^{\pi}$ is the optimal value function with respect to the policy $\pi$ we are following.

\begin{rem}
	For the action value function, the Bellman Expectation Operator $\mathcal{T}^\pi_Q$ and the Bellman Optimality Operator $\mathcal{T}^*_Q$ as defined by the Expression \ref{equation : Bellman-Operator-for-Q} have exactly the same properties.\newline
	Note also that, the Bellman Expectation Operator, by analogy with the formula \ref{eq:QvalueGen} is given by :
	\begin{equation}\label{eq:BellExpeOperQValues}
		\left(\mathcal{T}^\pi_Q f\right)(s,a) = r(s,a)+\gamma\cdot\sum_{s'}p(s'|s,a)\left(\sum_{a'}\pi(a'|s')\cdot f(s',a')\right)
	\end{equation}
\end{rem}
Here we have presented in a condensed way, the Bellman Operators. Now, we are able to say precisely how the Reinforcement Learning algorithms work really. This is why, in the following section, we are talking about the policy evaluation and policy iteration methods but now in the operator setting and knowing all the properties of those operators, it shows us why the RL algorithm come to work.
\section{Policy evaluation and iteration in operators setting}
From these properties (of the Bellman Operator), we can now deduce two ways of searching for the state-value function and the best policy. There are the well-known foundational \textbf{policy evaluation} and \textbf{policy improvement} (but this last one will be illustrated through policy iteration) algorithms (as presented in the previous chapter) :
\begin{enumerate}
	\item \textbf{Policy evaluation :}\newline
	The policy evaluation, is a method used to find the true state-value function associated to an MDP while following a specific policy. It's based on the properties of the Bellman operator. Starting from a random distribution of value functions, until reaching the true one as follows :
	\begin{itemize}
		\item Start with $v_0$ a random initial value function
		\item  Update the value function using : $v_{k+1}\leftarrow T^{\pi}v_k$
	\end{itemize}
	By the Banach Fixed-point Principle, it will converge to $v_{\pi}$ as $k\rightarrow \infty$
	\item \textbf{Policy iteration :}\newline
	The policy iteration algorithm is used to find the best policy, by updating concurrently the policy (policy improvement) as well as the associated value-functions. It is done in the following way :
	\begin{itemize}
		\item Start with a random policy $\pi_0$
		\item Concurrent iterations :
		\begin{itemize}
			\item[*] Policy evaluation : $v_{k+1}\leftarrow T^{\pi_i}v_k$
			\item [*] Policy improvement : Greedily improve the policy using : $\pi_{i+1} = \underset{a\in\mathcal{A}(s)}{argmax~}q_{\pi_i}(s,a)$
		\end{itemize}
	\end{itemize}
	Again, using the Banach Fixed-point Principle, it will converge to the best policy $\pi$ and the associated value function $v_\pi$ as $k\rightarrow \infty$.
\end{enumerate}
This is sufficient to justify how Reinforcement Learning algorithms work because it explains how the two foundational methods work.
\section*{Chapter Conclusion}
In this chapter we have talked about Bellman operators. We first presented the Bellman optimality operators (for both the state and the state-action value functions), before presenting a general formulation via the Bellman expectation operators (again for both the state and the state-action value functions). We have shown that these operators are $\gamma-$contraction and monotonic mappings, i.e. they have only one fixed point and that fixed point is the optimum. This is very important because it justifies why and how RL algorithms work from a fundamental point of view. In the next chapter we will now discuss some limitations of the Bellman operators in their classical expression and show some alternatives that seem interesting, even if they require further investments.
\chapter{Contributions, implementations and analysis}\label{chap4}
In this chapter we discuss some limitations of the classical Bellman operators as introduced in the previous chapter. We present how we have to make a trade-off between optimality and efficiency, and finally we show some practical results that we have obtained through experiments, which suggest that we need to think more about some refinements that should be made either to the Bellman operators or to the associated value functions.
\section{Alternatives to the Bellman Operator}
As we have seen, value-based reinforcement learning algorithms solve decision making problems through iterative application of a convergent operator, and an initial value function is recursively improved. Many other research works have examined alternatives to the Bellman operator \cite{asadi2017alternative,azar2011speedy, bellemare2016increasing, bertsekas2012q,lu2018general, watkins1989learning}, but there are two which seem to be more promising, this is the reason why we will take inspiration from them. The first one is the \textbf{consistent Bellman Operator} \cite{bellemare2016increasing}, and the second one is the family of \textbf{Robust Stochastic Operators} \cite{lu2018general}. These two operators appear to be more interesting than those suggested before \cite{lu2018general}, but for the second, \textbf{we'll suggest a refinement of it which will remain good (we postulate) even without the stochasticity, improving the stability}.
\subsection{Motivation \cite{bellemare2016increasing, lu2018general}}
The motivation for new expressions of the Bellman Operator is simple. While Q-learning based methods continue to be successfully used in RL to determine the optimal policy,  there is always a need to improve the convergence speed, accuracy and robustness of these algorithms. Now considering the fact that most of the time there are intrinsic error approximation (for example in the usual case of using discrete MDP to approximate a continuous system), the optimal state-action value function obtained through the Bellman
operator does not always describe the value of stationary policies. And most importantly, when the differences between the
optimal state-action value function and the suboptimal state-action value functions are small, this can lead to errors in identifying the true optimal actions. This is why, even if the classical Bellman Operators can solve problems in a perfect discrete setting, we need to refine it to be general because we know that for many real-world problems of our interest, discrete-time systems come from approximating continuous systems. So, there is always an intrinsic error, thus a need to integrate the corrector in the Operator.\newline
Let us now study the effectiveness of the two aforementioned operators and find some insights on the general method of determining optimal policies in RL.
\subsection{The Consistent Bellman Operator}
This operator, is primarily defined by the authors of \cite{bellemare2016increasing} for the action-value function. Here is the definition :
\begin{equation}\label{ConsistentBellman}
	\mathcal{T}_c f(s,a) = r(s,a) + \gamma\cdot\sum_{s'}p(s'|s,a)\cdot \left[\mathbb{I}_{\{s\neq s'\}}\underset{a'}{\max}~f(s',a')+\mathbb{I}_{\{s= s'\}}f(s,a)\right] ~~~ \text{with $f\in \mathcal{Q}$}
\end{equation}
Knowing that $\mathbb{I}$ is the indicator function, we claim also that, the following properties holds for the consistent Bellman Operator given by the Expression \ref{ConsistentBellman} :
\begin{enumerate}
	\item $\mathcal{T}_c$ is a contraction mapping,
	\item $\mathcal{T}_c$ is monotonic.
\end{enumerate}
\begin{proof}
	First, let's do some refinements :
	\begin{itemize}
		\item We will replace for simplicity 
		\[\sum_{s'}p(s'|s,a)\cdot(f) \text{~~~~by~~~~} \mathbb{E}_\mathbb{P} (f)\]
		\item We will refine the action-value function as follows :
		\begin{equation}\label{QRefinement}
			f_s(s',a') = \left\{\begin{array}{lcl}
				f(s',a') & \text{if} & s\neq s'\\
				f(s,a) & \text{if} & s= s'\\
			\end{array} ~~~\text{for}~~~f\in \mathcal{Q}\right.
		\end{equation}
	\end{itemize}
	Now we can rewrite the consistent operator in \ref{ConsistentBellman} as follows :
	\begin{equation}\label{ConsistentBellmanRef}
		\mathcal{T}_c f(s,a) = r(s,a) + \gamma\cdot\mathbb{E}_\mathbb{P}\left[\max_{a'}~f_s(s',a')\right] ~~~ \text{with $f\in \mathcal{Q}$}
	\end{equation}
	We can now start the proof :
	\begin{enumerate}
		\item Let $u, v \in \mathcal{Q}$, we will have :
		\begin{align}
			\left|\mathcal{T}_c u(s,a) - \mathcal{T}_c v(s,a)\right| &= \left| \gamma\cdot\mathbb{E}_\mathbb{P}\left(\underset{a'}{\max}~u_s(s',a')\right)-\gamma\cdot\mathbb{E}_\mathbb{P}\left(\underset{b'}{\max}~v_s(s',b')\right) \right|\nonumber\\
			&= \gamma\cdot\left| \mathbb{E}_\mathbb{P}\left(\underset{a'}{\max}~u_s(s',a')\right)-\mathbb{E}_\mathbb{P}\left(\underset{a'}{\max}~v_s(s',a')\right) \right|\nonumber\\
			&\leq \gamma\cdot\left| \underset{a'}{\max}~\mathbb{E}_\mathbb{P}\left(u_s(s',a')-v_s(s',a')\right) \right|\nonumber\\
			&\leq \gamma\cdot\underset{a'}{\max}~\left| \mathbb{E}_\mathbb{P}\left(u_s(s',a')-v_s(s',a')\right) \right|\nonumber\\
			&\leq \gamma\cdot\underset{s',a'}{\max}~\left|u_s(s',a')-v_s(s',a') \right|\nonumber\\
			&= \gamma\cdot\underset{s,a}{\max}~\left|u_s(s,a)-v_s(s,a) \right|\nonumber\\
			&= \gamma\cdot||u_s(s,a)-v_s(s,a) ||_\infty\nonumber\\
			\Rightarrow \left|\mathcal{T}_c u(s,a) - \mathcal{T}_c v(s,a)\right| &\leq \gamma\cdot||u_s(s,a)-v_s(s,a) ||_\infty\nonumber\\
			\Rightarrow ||\mathcal{T}_c u(s,a) - \mathcal{T}_c v(s,a)||_\infty &\leq \gamma\cdot||u_s(s,a)-v_s(s,a) ||_\infty
		\end{align}
		\item Now, for us to prove the monotonicity, we have to consider two state-action value functions $u$ and $v$ such that $u(s,a)\leq v(s,a) ~~\forall (s,a)\in \mathcal{S}\times\mathcal{A}$. From the definition we gave through the Expression \ref{QRefinement}, this also means that $u_s(s,a)\leq v_s(s,a)$. Based now on the proof we just provided for the monotonicity, we have :
		\begin{align}
			\mathcal{T}_c u(s,a) - \mathcal{T}_c v(s,a) &\leq \gamma\cdot\underset{a'}{\max}~\left( \mathbb{E}_\mathbb{P}\underset{u_s\leq v_s}{\underbrace{\left(u_s(s',a')-v_s(s',a')\right)}}\right)\nonumber\\
			\Rightarrow \mathcal{T}_c u(s,a) - \mathcal{T}_c v(s,a) &\leq 0 \nonumber\\
			\Rightarrow \mathcal{T}_c u(s,a) &\leq \mathcal{T}_c v(s,a) ~~~\forall~ u,v\in \mathcal{Q}
		\end{align}
	\end{enumerate}
\end{proof}
From this proof, we can conclude that $\mathcal{T}_c$ has only one fixed point and that fixed point is optimal for this setup (that is the setup in which we are using the consistent Bellman equation instead of the classical one).\newline
\textbf{The question that remains is how this fixed point relates to the one we have found using the classical Bellman operator.}\newline
The most important thing is that, we are sure to converge to a unique fixed-point. As we don't yet have the mathematical tool to say something about the relation between that fixed point and the one which will be found using the classical Bellman Operator, we can just wait and see how the operator will behave in practice in comparison to the Bellman Operator.
\subsection{Modified Robust Stochastic Operator}
Here, the proposed operator take inspirations in both the articles \cite{lu2018general} and \cite{bellemare2016increasing} but mostly that former article. This is more general (expectation operator) and different from the Robust Stochastic Operator suggested in \cite{lu2018general}.\newline
We know that, as an RL algorithm get better, if we express $v_\pi(s)$ by :
\[v_\pi(s) = \sum_{a}\pi(a|s)\cdot q_\pi(s,a),\]
the difference that we will call \textbf{advantage learning}, given by
\[A(s,a) = q_\pi(s,a)-v_\pi(s),\]
always gives an idea on how well is our policy (which should be learning how to often choose better actions) as well as the value function. The quantity $\left|A(s,a)\right|$ should get smaller as we are learning the good policy. \newline
Now, let's say we modify the \textbf{Bellman Expectation Operator for the action-value function} as defined by the Relation \ref{eq:BellExpeOperQValues} in such a way that we integrate the idea of the \textbf{advantage learning} directly in the Bellman Operator instead of waiting applying it in the learning process as for some policy gradient methods \cite{graesser2019foundations}. Now, if we call that new operator $\mathcal{T}_a$, it will be defined for all $f\in\mathcal{Q}$ as follows :
\begin{equation}\label{AdvantageOperator}
	\left(\mathcal{T}_a f\right) = r(s,a)+\gamma\cdot\sum_{s'}p(s'|s,a)\left(\sum_{a'}\pi(a'|s')\cdot f(s',a')\right) + \beta\cdot\underset{A(s,a)\text{:~advantage learning}}{\underbrace{\left[f(s,a)-\sum_{a}\pi(a|s) f(s,a)\right]}}
\end{equation}
We can now study the properties of the operator defined by this Expression \ref{AdvantageOperator}. The coefficient $\beta$ is any real now but we will define it properly at the end of this theoretical investigation.\newline
\begin{pro}
	The operator $\mathcal{T}_a$ as defined by the Relation \ref{AdvantageOperator} is not a contraction mapping.
\end{pro}
\begin{proof}
	Let $u(s,a)\equiv u$ and $v(s,a)\equiv v$ be two elements of $\mathcal{Q}$ :
	\begin{align}
		\left|\mathcal{T}_a u - \mathcal{T}_a v\right| &= \left|\gamma\mathbb{E}_\mathbb{P}\left(\sum_{a'}\pi(a'|s')\left(u(s',a')-v(s',a')\right)\right) + \beta\left[(u-v)-\sum_{a}\pi(a|s)(u-v)\right]\right|\nonumber\\
		&> \beta\cdot\Big|u(s,a)-v(s,a)\Big| ~~~\text{for some couples $(u,v)$~ and ~ $\beta$}\nonumber\\
		\Rightarrow ||\mathcal{T}_a u - \mathcal{T}_a v||_\infty &> \beta\cdot||u-v||_\infty
	\end{align}
	So, as we can see, whatever $\beta$ will be, there exist situations where the operator $\mathcal{T}_a$ is not a contration.
\end{proof}
Since the operator $\mathcal{T}_a$ is not a contraction, we cannot assess its convergence. But, even if we are not sure about the convergence, we can still say something about its behavior. For that, inspired by the definitions given in \cite{lu2018general,bellemare2016increasing}, let's define two important properties that we will call the properties of a \textbf{well-behaving operator}. There are the optimality preservation and the gap increasing properties.\newline
\textbf{Note} : in the following definitions, we'll change a little bit our conventions by writing value functions in capital letters. This will help us to write properly operators as indices.
\begin{defn}[Optimality preservation]\label{OptimPreserv}
	Let call $\mathcal{T}_a$ an alternative to the Bellman operator and $\mathcal{T}_b$ the classical Bellman operator. We will say that $\mathcal{T}_a$ preserves the optimality property if :
	\[Q_{k,\mathcal{T}_b}<V_{k,\mathcal{T}_b}\Rightarrow Q_{k,\mathcal{T}_a}<V_{k,\mathcal{T}_a}~~~~\text{when $k\to \infty$}\]
	$Q_{k,\mathcal{T}}$ represents the action value function that we find at the iteration $k$, using the operator $\mathcal{T}$ and $V_{k,\mathcal{T}}$ the associated state value function.
\end{defn}
\begin{defn}[Gap increasing]\label{GapIncreasing}
	Using the same notation as previously, we will say that an alternative operator $\mathcal{T}_a$ to the Bellman one $\mathcal{T}_b$ induces the gap increasing property if, for each state $s\in\mathcal{S}$ and each feasible action in that state $a\in\mathcal{A}(s)$ :
	\[\left|\lim\limits_{k\to\infty}\left[Q_{k,\mathcal{T}_b} - V_{k,\mathcal{T}_b}\right]\right| \leq \left|\lim\limits_{k\to\infty} \left[Q_{k,\mathcal{T}_a} - V_{k,\mathcal{T}_a}\right]\right|\]
	Knowing that at each state, $V_{k,\mathcal{T}}$ is an average of all $Q_{k,\mathcal{T}}$ that we can get by making a certain action from that state.
\end{defn}
The \textbf{optimality preservation} and the \textbf{gap increasing} properties represent respectively :
\begin{itemize}
	\item How well the operator is likely to behave while looking for the fixed point and,
	\item How is the model able to differentiate the values of suboptimal actions with those for optimal actions.
\end{itemize}
\begin{pro}
	We claim that, the operator defined by the Expression \ref{AdvantageOperator}, even if it is not a contraction, is both optimality preserving and gap increasing.
\end{pro}
\begin{proof} Let's prove those properties in the same order :
	\begin{enumerate}
		\item \textbf{Optimality preservation} : We know that $Q_{k,\mathcal{T}} = \mathcal{T}Q_{k-1}$ and $V_{k, \mathcal{T}_b} = \displaystyle\sum_a \pi(a|s). Bell(Q_{k-1})$. Now, using the Definition \ref{OptimPreserv}, we assume that $Q_{k,\mathcal{T}_b}<V_{k,\mathcal{T}_b}$ and we will deduce the relation between $Q_{k,\mathcal{T}_a}$ and $V_{k,\mathcal{T}_a}$. To make the notation easy to manage, we will do this replacement :
		\[r(s,a)+\gamma\cdot\mathbb{E}_\mathbb{P}\left(\sum_{a'}\pi(a'|s')\cdot Q_{k-1}(s',a')\right) = Bell\left(Q_{k-1}\right)\]
		Now we can start (knowing that we are working with $k\to\infty$) :
		\begin{align}
			Q_{k,\mathcal{T}_b} &< V_{k,\mathcal{T}_b}\nonumber\\
			\Rightarrow \mathcal{T}_bQ_{k-1} &< V_{k,\mathcal{T}_b}\nonumber\\
			\Rightarrow Bell\left(Q_{k-1}\right) &< \sum_{a}\pi(a|s)\cdot Bell\left(Q_{k-1}\right)\nonumber
            \end{align}
   without a loss of generality, let $K=\min_a{\Big(Q_{k-1}(s,a)-V_{k-1}(s)\Big)}$. Then
            \begin{align}
            Bell\left(Q_{k-1}\right) &< \sum_{a}\pi(a|s)\cdot Bell\left(Q_{k-1}\right)\nonumber\\
			\Rightarrow Bell\left(Q_{k-1}\right) + \beta\cdot K &< \sum_{a}\pi(a|s)\cdot Bell\left(Q_{k-1}\right)+\beta\cdot K\cdot1\nonumber\\
			\Rightarrow Bell\left(Q_{k-1}\right) + \beta\cdot K &< \sum_{a}\pi(a|s)\cdot Bell\left(Q_{k-1}\right)+\beta\cdot K\cdot\sum_{a}\pi(a|s)\nonumber\\
			\Rightarrow Bell\left(Q_{k-1}\right) + \beta\cdot K &< \sum_{a}\pi(a|s)\cdot Bell\left(Q_{k-1}\right)+\beta\cdot\sum_{a}\pi(a|s)\cdot K\nonumber
            \end{align}
            we know that $K\leq\sum_{a}\pi(a|s)\cdot\left(Q_{k-1}(s,a)-V_{k-1}(s)\right) ~~\forall~ (s,a)  \in \mathcal{S}\times\mathcal{A}$, and when $k\to\infty$ as we are getting better policy, we have almost surely that $\left(Q_{k-1}-V_{k-1}\right)\leq \sum_{a}\pi(a|s)\cdot\left(Q_{k-1}-V_{k-1}\right)$
            \begin{align}
                Bell\left(Q_{k-1}\right) + \beta\cdot K &< \sum_{a}\pi(a|s)\cdot Bell\left(Q_{k-1}\right)+\beta\cdot\sum_{a}\pi(a|s)\cdot K\nonumber\\
			\Rightarrow Bell\left(Q_{k-1}\right) + \beta\cdot \left(Q_{k-1}-V_{k-1}\right) &< \sum_{a}\pi(a|s)\cdot Bell\left(Q_{k-1}\right)+\beta\cdot\sum_{a}\pi(a|s)\cdot \left(Q_{k-1}-V_{k-1}\right)\nonumber\\
			\Rightarrow \underset{Q_{k,\mathcal{T}_a}}{\underbrace{Bell\left(Q_{k-1}\right) + \beta\cdot \left(Q_{k-1}-V_{k-1}\right)}} &< \sum_{a}\pi(a|s)\cdot\Big[ Bell\left(Q_{k-1}\right)+\beta\cdot \left(Q_{k-1}-V_{k-1}\right)\Big]\nonumber\\
			\Rightarrow Q_{k,\mathcal{T}_a} &< \sum_{a}\pi(a|s)\cdot Q_{k,\mathcal{T}_a}(s,a)\nonumber\\	
			\Rightarrow Q_{k,\mathcal{T}_a} & < V_{k,\mathcal{T}_a}	~~~~~~~~~\text{as $k\to\infty$}					
		\end{align}
		\item \textbf{Gap increasing} : Using the Definition \ref{GapIncreasing}, we have now to show that :
		\[\left|\lim\limits_{k\to\infty}\left[Q_{k,\mathcal{T}_b} - V_{k,\mathcal{T}_b}\right]\right| \leq \left|\lim\limits_{k\to\infty} \left[Q_{k,\mathcal{T}_a} - V_{k,\mathcal{T}_a}\right]\right|\]
		In fact, we know that, using the Bellman Operator $\mathcal{T}_b$, if we are improving the policy at the same time, the state value function $V_{k,\mathcal{T}_b}$ will converge to the action value function $ Q_{k,\mathcal{T}_b} $ as the policy get better. So, for optimal actions we will have :
		\[\lim\limits_{k\to\infty} \left[Q_{k,\mathcal{T}_b} - V_{k,\mathcal{T}_b}\right] = 0,\]
		i.e. also using $\mathcal{T}_a$, we can say surely for optimal actions :
		\[\left|\lim\limits_{k\to\infty}\left[Q_{k,\mathcal{T}_b} - V_{k,\mathcal{T}_b}\right]\right| = 0 ~~~~\leq~~~~ \left|\lim\limits_{k\to\infty} \left[Q_{k,\mathcal{T}_a} - V_{k,\mathcal{T}_a}\right]\right|\]
  For other actions, the inequality also holds because on convergence $\left(Q_{k-1}-V_{k-1}\right) < 0$ almost surely and from the proof of optimality preservation we have seen that :\newline
  $_{}$\newline
  $Bell\left(Q_{k-1}\right) + \beta\cdot \left(Q_{k-1}-V_{k-1}\right) < \sum_{a}\pi(a|s)\cdot Bell\left(Q_{k-1}\right)+\beta\cdot\sum_{a}\pi(a|s)\cdot \left(Q_{k-1}-V_{k-1}\right),$
  i.e. also $\left|\lim\limits_{k\to\infty}\left[Q_{k,\mathcal{T}_b} - V_{k,\mathcal{T}_b}\right]\right| \leq \left|\lim\limits_{k\to\infty} \left[Q_{k,\mathcal{T}_a} - V_{k,\mathcal{T}_a}\right]\right|$ for any $(s,a)\in \mathcal{S}\times\mathcal{A}$.
	\end{enumerate}
\end{proof}
So, the operator that we suggested through the Expression \ref{AdvantageOperator} is a \textbf{well behaving operator}, even if we just assume the convergence (modulo the value of $\beta$).\newline
Let us now give some comments about the possibility of finding a fixed point for this operator. But, before doing so, we can first say something about the continuity and the boundedness of that operator because without that, we cannot talk about fixed points.
\begin{pro}
	Let $r(s,a)$, the reward function, be bounded and continuous. The operator $\mathcal{T}_a$ as defined by the Expression \ref{AdvantageOperator}, will also be \textbf{bounded} and \textbf{continuous}.
\end{pro}
\begin{proof}
	We can split our operator in three parts for simplicity :
	\begin{equation*}
		\left(\mathcal{T}_a f\right) = \underset{\text{part 1}}{\underbrace{r(s,a)}}+\underset{\text{part 2}}{\underbrace{\gamma\cdot\sum_{s'}p(s'|s,a)\left(\sum_{a'}\pi(a'|s')\cdot f(s',a')\right)}} + \underset{A(s,a)\text{:~advantage learning :: part 3}}{\underbrace{\beta\cdot\left[f(s,a)-\sum_{a}\pi(a|s) f(s,a)\right]}}
	\end{equation*}
	\begin{enumerate}
		\item \textbf{Part 1} : the reward function is assumed to be continuous and bounded.
		\item \textbf{Part 2} : If we take that part as an operator by itself, it will be bounded (because it is a contraction) and also linear. And being a bounded linear operator implies that it is also continuous. So, if we call this new operator $\mathcal{T}_{small}$ written as follows :
		\[\Big(\mathcal{T}_{small} f\Big)(s,a) =\gamma\cdot\sum_{s'}p(s'|s,a)\left(\sum_{a'}\pi(a'|s')\cdot f(s',a')\right) = \gamma\cdot\mathbb{E}_\mathbb{P} \left(\sum_{a'}\pi(a'|s')\cdot f(s',a')\right),\]
		We can firstly prove the linearity, assuming that $u,v\in \mathcal{Q}$ and $\alpha,\beta\in \mathbb{R}$. We have to show that $$\mathcal{T}_{small}\Big(\alpha\cdot u+\beta\cdot v\Big) = \alpha\cdot\mathcal{T}_{small}~u+\beta\cdot\mathcal{T}_{small}~v $$
		And, that is true because the operator is just formed with average operations, which are in fact linear. So, $\mathcal{T}_{small}$ is linear.\newline
		Let's now prove the contraction. For that, using again $u$ and $v$ defined previously, we have :
		\begin{align}
			\Big|\mathcal{T}_{small}~u-\mathcal{T}_{small}~v\Big| &= \Big|\gamma\cdot\mathbb{E}_\mathbb{P} \sum_{a'}\pi(a'|s')\cdot \left(u(s',a')-v(s',a')\right)\Big|\nonumber\\
			&\leq \gamma\cdot\mathbb{E}_\mathbb{P}\Big|\sum_{a'}\pi(a'|s')\cdot \left(u(s',a')-v(s',a')\right)\Big|\nonumber\\
			&\leq \gamma\cdot\mathbb{E}_\mathbb{P}\sum_{a'}\pi(a'|s')\cdot\Big| \left(u(s',a')-v(s',a')\right)\Big|\nonumber\\
			&\leq \gamma\cdot\underset{s',a'}{\max}\Big| \left(u(s',a')-v(s',a')\right)\Big|\nonumber\\
			&= \gamma\cdot\underset{s,a}{\max}\Big| \left(u(s,a)-v(s,a)\right)\Big|\nonumber\\
			\Rightarrow \Big|\mathcal{T}_{small}~u-\mathcal{T}_{small}~v\Big| &\leq \gamma\cdot\underset{s,a}{\max}\Big| \left(u(s,a)-v(s,a)\right)\Big|\nonumber\\
			\Rightarrow \underset{s,a}{\max}\Big|\mathcal{T}_{small}~u-\mathcal{T}_{small}~v\Big| &\leq \gamma\cdot\underset{s,a}{\max}\Big| \left(u(s,a)-v(s,a)\right)\Big|\nonumber\\
			\Rightarrow ||\mathcal{T}_{small}~u-\mathcal{T}_{small}~v||_\infty &\leq \gamma\cdot|| \left(u(s,a)-v(s,a)\right)||_\infty
		\end{align}
		Thus, the second part is also continuous and bounded.
		\item \textbf{Part 3} : We know that, if $f$ is continuous and bounded, then $f(s,a)-\sum_{a}\pi(a|s) f(s,a)$ will also be continuous and bounded. We just need to ensure that $\beta$ has good properties, ensuring convergence.
	\end{enumerate}
	Finally, knowing that the summation of continuous and bounded operators is also continuous and bounded, we have proved the assertion, and $\mathcal{T}_a$ is continuous and bounded (the notion of good behavior for $\beta$ is just to be precised).
\end{proof}
From the fact that $\mathcal{T}_a$ is continuous and bounded (modulo the behavior of $\beta$), we are sure that it will not diverge if we set well the values of $\beta$. And, looking carefully on that operator, we can now say something about $\beta$.\newline
We will always need an operator which is well behaving (improving speed and optimality) but also close enough to the classical Bellman Operator as we are evolving. That is the only way we can ensure that we will get closer enough to the \textbf{same fixed point} while maintaining some nice properties of this alternative operator.\newline
If now we consider a family of operators based on $\mathcal{T}_a$ by varying $\beta$, and the case of a $\beta$ which changes per iteration in one operator during the learning process, to ensure convergence, they should fulfill the following conditions :
\begin{equation}\label{eq:ConditionsBeta}
	\left\{ \begin{array}{lcl}
		\sum\limits_{j=1}^{\infty} \beta_{i,j} &<& \infty\\
		&&\\
		\{\beta_{i,j}\} &\to& 0 ~~~\text{with $j$ fixed.}
	\end{array} \right. ~~\text{$i$ : the operator,~~ $j$ : iteration}
\end{equation}
For an operator which has a varying $\beta$ across training, as it will be applied iteratively, we can directly notice that after many iterations, we will have some terms like this
\[\sum\beta_{i,j}\cdot\left(u(s,a)-v(s,a)\right)_j \equiv \sum\beta_{i,j}\cdot A_j,\]
and as the difference $\left(u(s,a)-v(s,a)\right)_j$ is bounded, we can replace it by its bound. Let's call the bound of  $\left(u(s,a)-v(s,a)\right)_j\equiv A_j$ simply $A$. Then :
\[\sum\beta_{i,j}\cdot A_j \leq A\cdot\sum_{j}\beta_{i,j} < \infty \iff \sum_{j}\beta_{i,j}< \infty\]
So, for one operator, with varying $\beta$ across iterations, the sum of values should be finite.\newline
Also, with the conditions we gave, it allows us to construct directly a sequence of operators converging to the classical Bellman Operator. So, that's why we need :
\[\{\beta_{i,j}\} \to 0 ~~~\text{with $j$ fixed.}\]
\begin{exa}
	One example of a family of operators with the form given by the Expression \ref{AdvantageOperator} and fulfilling the conditions on $\beta$ given by \ref{eq:ConditionsBeta}, can be :
	\begin{enumerate}
		\item Operator 1 : $\beta$ is a sequence given by $\{\beta_{1,j}\}$ with $\beta_{1,j}=\dfrac{1}{j^2}$
		\item Operator 2 : $\beta$ is a sequence given by $\{\beta_{2,j}\}$ with $\beta_{2,j}=\dfrac{1}{j^3}$
		\item Operator $k$ : $\beta$ is a sequence given by $\{\beta_{k,j}\}$ with $\beta_{k,j}=\dfrac{1}{j^{k+1}}$.
	\end{enumerate}
\end{exa}
As last comment on this part we can observe : during our diverse proofs, we just require to the mappings to be contractions and \textbf{monotonic} in order to establish the existence of an \textbf{optimal} fixed point. The use of specific definitions of value functions was not really important. This is why we can conjecture that :
\begin{conj}
	Any monotonic contraction mapping integreting a certain notion of policy is a suitable candidate as operator in Reinforcement Learning (we just have to refine accordingly the definitions of the associated value functions for them to be suitable to the RL framework).
\end{conj}
\section{Practical implementation and analysis}
To prove the effectiveness of the suggested \textit{Modified Robust Stochastic Operator}, we have conducted experiments on 3 groups of classical problems in reinforcement learning using the \textbf{Q-Learning} algorithm, in \textbf{Open-AI Gymnasium environments}. The implementation of Q-Learning we used is inspired by \cite{GithubVmayoral} and our own implementation with all the parameters is in our Github Repository \cite{MyGithub}.\newline
Here are the necessary details on the environments we used :
\subsection{Environments and results}
\begin{itemize}
	\item[1.] Mountain Car environment :
	The theory about this environment is presented in \cite{moore1990efficient}. The state vector is 2-dimensional, continuous with a total of three possible actions. As long as the goal is not yet reached, depending on the action, a negative reward is given to the agent, until it reaches the goal.  Following \cite{lu2018general}, we have discretized the state space into a $40\times 40$ grid, but differently we did 10000 training steps, with 10000 episodes
each. The following Figure \ref{fig:MountainCar} shows the averages across episodes.
	\begin{figure}[htbp!]
		\centering
		\includegraphics[width=0.6\textwidth]{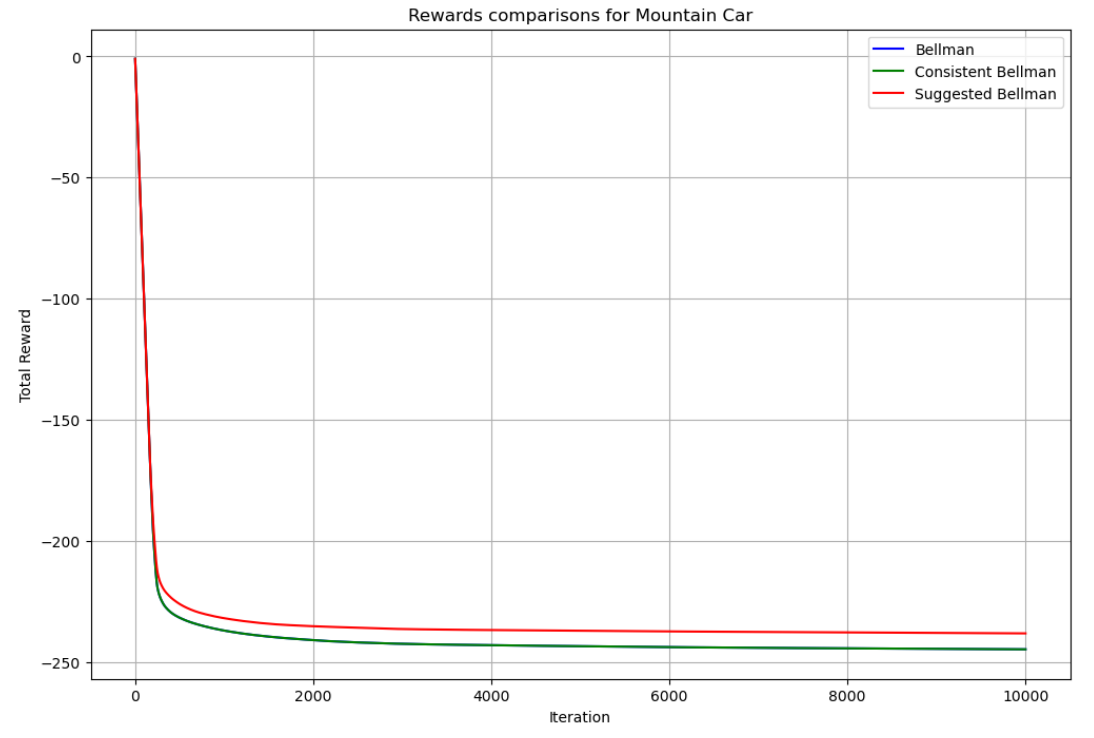}
		\caption{The 3 operators on MountainCar environment.}\label{fig:MountainCar} 
	\end{figure}
	\item[2.] Cart Pole environment :
	The theory about Cart Pole is presented in \cite{barto1983neuronlike}. The state vector is 4-dimensional and continuous with a total of two possible actions. The aim is to keep the pole upright for as long as possible, with a reward of $+1$ for each step up to the failure, including the final step. So, the reward is positive at the end.  Here, we have discretized the state space into a $150\times 150\times 150 \times 150$ grid and again we did 10000 training steps, with 10000 episodes each. The following Figure \ref{fig:CartPole} shows the averages across episodes.
	\begin{figure}[htbp!]
		\centering
		\includegraphics[width=0.6\textwidth]{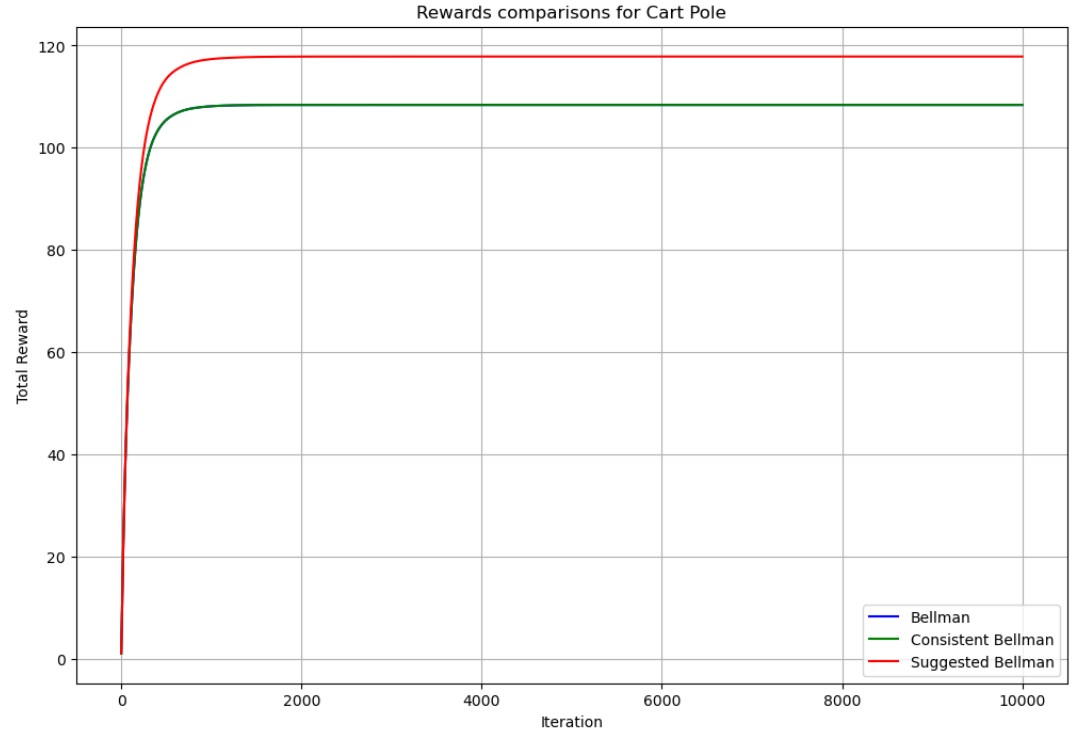}
		\caption{The 3 operators on CartPole environment.}\label{fig:CartPole}
	\end{figure}

	\item[3.] Acrobot environment :
	The theory about Acrobot is presented in \cite{sutton1995generalization}. The state vector is 6-dimensional and continuous with a total of three possible actions. The goal is to have the free end of the Acrobot to reach the target height (represented by a horizontal line) in as few steps as possible, with each step not reaching the target being rewarded with -1. So, the reward is negative again as for Mountain Car.  Here, we have discretized the state space into a $30\times 30\times 30 \times 30\times 30 \times 30$ grid, due to the limitations in memory allocation of the computer we was using, and again we did 10000 training steps, with 10000 episodes each. The following Figure \ref{fig:Acrobot} shows the averages across episodes.
	\begin{figure}[htbp!]
		\centering
		\includegraphics[width=0.6\textwidth]{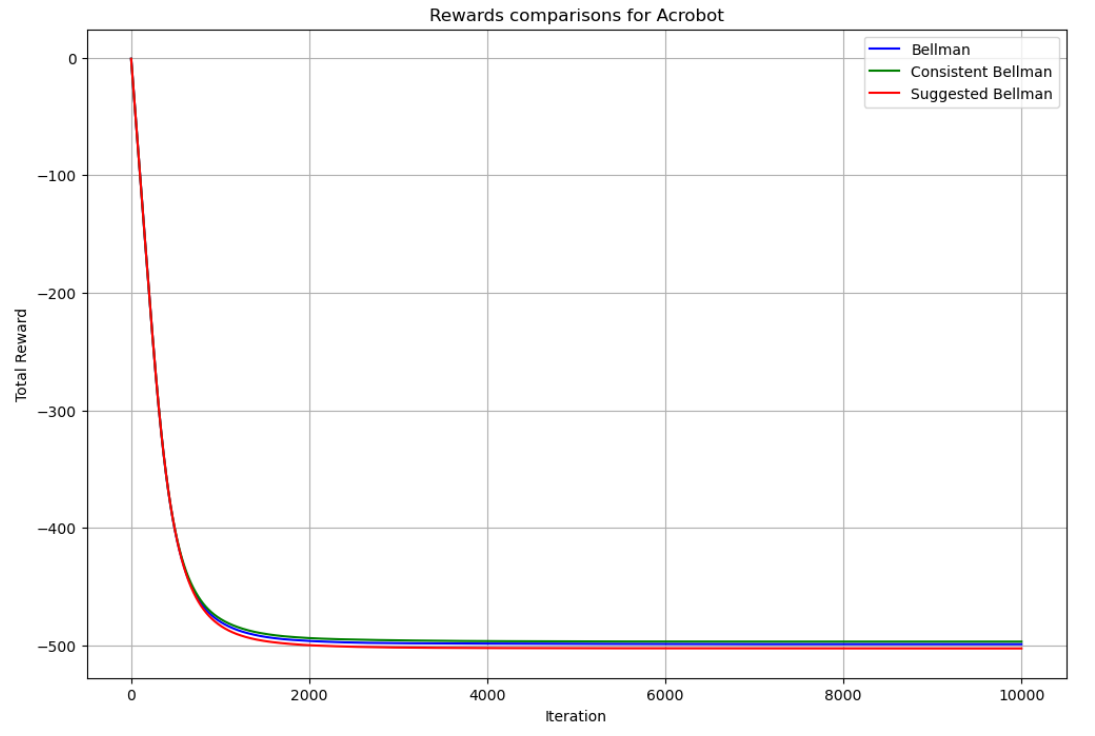}
		\caption{The 3 operators on Acrobot environment.}\label{fig:Acrobot}
	\end{figure}
\end{itemize} \newpage
\subsection{Interpretation and discussion}
\begin{itemize}
	\item[1.] Mountain Car : From the problem presented above, we know that the reward for mountain car can be negative because it depends with how long the agent takes to reach the goal or to be stopped. Now, looking on the Figure \ref{fig:MountainCar} we can directly see how the average reward is better using the \textit{Modified Robust Stochastic Operator} than using the  Bellman Operator. Also, the classical and the consistent operators are performing exactly the same for this specific case. So, overall, the suggested Operator finishes the episodes with higher reward than for the other two operators (the classical and the consistent).
	\item[2.] Cart Pole : For Cart Pole, we can see in Figure \ref{fig:CartPole} that the higher reward is reached again while using the Modified Bellman (suggested Bellman) compared to the other operators. And the difference in terms of performances is really clear.
	\item[3.] Acrobot : Acrobot was more challenging and as we can see in Figure \ref{fig:Acrobot}, the results using those 3 operators are about the same. So, we think this is due to the fact that we were not able to make a finer discretization for the Acrobot environment. In fact this was due to the memory limitations. So, this experiment needs further investigations to establish clearly what is the great attainable difference between those operators.
\end{itemize}
So, generally we can say that the Consistent Bellman Operator gives about the same result as the classical one, but our modified version of the Robust Stochastic Operator gives most of the time better results comparing to the other previous two operators.
\section*{Chapter Conclusion}
In this last chapter, we have introduced the Consistent Bellman Operator as the first alternative to the classical Bellman Operator presented in the previous chapter, and we showed that it fulfills exactly the same properties as the classical one (uniqueness of the optimal fixed point), although this raises the question of the relation between a fixed point found by the Consistent Bellman and the one we can find by the classical Bellman. Secondly, we presented our version of the robust stochastic operator. We modified the one proposed in the literature in such a way that it no longer needs stochasticity, but still preserves some good properties (gap increasing and optimality preservation), even if it loses the properties of a classical Bellman operator. Finally, with results from a practical implementation, we clearly show that our proposed operator must be taken seriously, since it unexpectedly outperforms the classical Bellman operator on some difficult tasks, as hypothesised.
\chapter{Conclusion}

%The following work has served as the first stone on studying deeply and in a precise mathematical way Reinforcement Learning concepts. This is done in order to facilitate breakthrough discoveries or investigations after, and to allow even pure mathematicians to work on the same subject and to suggest insightful adjustments, necessary for making better Reinforcement Learning algorithms. This has been done through different chapters in this work as follow :

This work serves as a foundational effort in the deep and precise mathematical study of Reinforcement Learning concepts. The aim is to facilitate breakthrough discoveries and subsequent investigations, and to enable even pure mathematicians to contribute to this subject by suggesting insightful adjustments necessary for improving Reinforcement Learning algorithms. This has been achieved through the various chapters in this work as follows:
 
%After the general introduction, in chapter \ref{chap1}, we have mainly discussed about contraction mappings in a metric space as well as the Banach Fixed Point Theorem. We have also shown how the theorem can be used to prove the existence and uniqueness of solutions for first order Ordinary Differential Equations with initial conditions.
After the general introduction, in Chapter \ref{chap1}, we primarily discussed contraction mappings in a metric space and the Banach Fixed Point Theorem. We also demonstrated how the Banach Fixed Point Theorem can be used to prove the existence and uniqueness of solutions to first-order Ordinary Differential Equations with initial conditions.

%In the chapter \ref{chap2} then, we have expressed Reinforcement Learning problem in terms of Markov Decision Process. This has now allow us to introduce different expressions of the value function (in terms of state and in terms of state-action pairs). After that, we have presented the concepts of Bellman Equations which now are the first step on introducing the Bellman Operators.
In Chapter \ref{chap2}, we expressed the Reinforcement Learning problem in terms of a Markov Decision Process. This allowed us to introduce different expressions of the value function, both in terms of state and state-action pairs. Subsequently, we presented the concepts of Bellman Equations, which served as the first step in introducing the Bellman Operators.

%The chapter \ref{chap3} were dedicated then especially to Bellman Operators. We have shown that those operators are $\gamma$-\text{contraction} and monotonic mappings, implying that they have only one fixed point and the mentioned fixed point is the optimum. This have justified why and how, in a fundamental point of view, Reinforcement Learning algorithms work.
Chapter \ref{chap3} was dedicated to Bellman Operators. We have demonstrated that these operators are $\gamma$-contraction and monotonic mappings, implying that they have a unique fixed point that is the optimum. This justification sheds light on why and how Reinforcement Learning algorithms work from a fundamental perspective.
%At the end, in chapter \ref{chap4}, we have presented the Consistent Bellman Operator as the first alternative to the classical Bellman Operator. After that, we have shown that it fulfill exactly the same properties as the classical Bellman Operator (uniqueness of the optimal fixed point) but this raised the question about the relation between a fixed point that we found through Consistent Bellman Operator and the one we can found through the classical Operator. We then presented a modification of the Robust Stochastic Operator which is in fact deterministic but defined in a more general and precise way, showing that there is no longer need of stochasticity to get better results that the classical Bellman Operator. At the end, with the results we got from an implementation in Python, using Open-AI Gymnasium environments, we have clearly show that our proposed operator needs to be taken seriously because it outperforms the classical Bellman Operator on many of the considered tasks.

At the end, in chapter \ref{chap4}, we presented the Consistent Bellman Operator as the first alternative to the classical Bellman Operator. After that, we showed that it fulfills exactly the same properties as the classical Bellman Operator (uniqueness of the optimal fixed point). However, this raised questions about the relationship between a fixed point found through the Consistent Bellman Operator and one found through the classical Operator. We then introduced a modification of the Robust Stochastic Operator, which is deterministic but defined in a more general and precise way. This demonstrates that stochasticity is no longer necessary to achieve better results than the classical Bellman Operator. Finally, using an implementation in Python and OpenAI Gymnasium environments, we clearly demonstrated that our proposed operator merits serious consideration, as it outperforms the classical Bellman Operator in many of the considered tasks.

%Based on the foundations laid by this work, the perspectives are many. The mathematical properties of value function spaces have been clearly presented in this work (in fact, value functions can be properly studied in Banach spaces), but we also need to deeply study the state space as well as the action space. We can even study the more challenging policy space. The most relevant attempts we have seen during our research are in \cite{ValuePolytope}, where the geometric and topological properties of value functions in finite-state action Markov decision processes are explored, and in \cite{MetricAndContinuityInRL}, where a unified formalism for defining state similarity metrics in Reinforcement Learning is introduced, hierarchies among metrics are established, and their implications for Reinforcement Learning algorithms are studied. We expect that they can be well understood and extended with the help of the investigations done in this work.
Based on the foundations laid by this work, the perspectives are numerous. The mathematical properties of value function spaces have been clearly presented (value functions can indeed be properly studied in Banach spaces), but future researchers can explore in depth the state space and the action space as well as the policy space. They can also extend the analysis of the efficiency of our proposed operator. The most relevant attempts we have seen during our research are in \cite{ValuePolytope}, where the geometric and topological properties of value functions for Markov decision processes are explored, and in \cite{MetricAndContinuityInRL}, where a unified formalism for defining state similarity metrics in Reinforcement Learning is introduced, hierarchies among metrics are established, and their implications for Reinforcement Learning algorithms are studied. We expect that these can be well understood and extended with the help of the investigations done in this work.

%This work can allow researchers, even those not involved in the field, to quickly and meaningfully engage with Reinforcement Learning, as the most important foundations are mathematically clarified.
This work enables researchers, even those not currently involved in the field, to quickly and meaningfully engage with Reinforcement Learning, as the most important foundations are mathematically clarified.

\appendix
\renewcommand*{\theHchapter}{\Alph{chapter}}
\endappendix
%-----------------------------------------------------------------------------
% See the acknowledgement.tex file and follow the instructions there.
\chapter*{Acknowledgements}
\addcontentsline{toc}{chapter}{Acknowledgements}
% Don't change anything above this.
% Overly long acknowledgements are not professsional.
% You are free to use your native language 

\begin{quote}
	``It is not that we think we are qualified to do anything on our own. Our qualification comes from God."
	\href{https://www.biblegateway.com/passage/?search=2%20Corinthians%203%3A5&version=NLT}{2 Corinthians 3:5}
\end{quote}
\begin{quote}
	``Transire suum pectus mundoque potiri. Deus ex machina.``
\end{quote}
	
%My sincere thanks to all the respectful individuals who have made AIMS such a beautiful and transformative experience for me, especially Professor Neil Turok, Professor Mama Foupouagnigni and Dr. Daniel Duviol. This work would not have been possible without the invaluable guidance of my supervisor, Dr. Ya\'{e} Ulrich Gaba, who, despite all his responsibilities, was always there to guide, advise, and help me, working with him was very amazing. My heartfelt thanks go to my mentor, Sir Domini Leko, whose simplicity, advice, and availability for help went beyond science and always inspire me. I can't forget to thank my beloved Sifa Shamavu Rose, who support me beyond measure and whose love and patience always remind me that I'm on the right path. Many thanks to all my friends at AIMS, especially Sinenkhosi Mamba, who made me feel at home. Thank you very much to all those who have been a constant support on my journey. May this work also serve as a testimony to the effort you have put into helping me, as well as the kindness you have shown me.

My sincere thanks to all the esteemed individuals who have made AIMS such a beautiful and transformative experience for me, especially Professor Neil Turok, Professor Mama Foupouagnigni, and Dr. Daniel Duviol. This work would not have been possible without the invaluable guidance of my supervisor, Dr. Ya\'{e} Ulrich Gaba, who, despite his many responsibilities, was always there to guide, advise, and help me; and working with him was truly amazing. My heartfelt thanks go to my mentor, Sir Domini Leko, whose simplicity, advice, and willingness to help went beyond science and always inspired me. I also extend my gratitude to my beloved Sifa Shamavu Rose, whose unwavering support, love, and patience constantly remind me that I am on the right path. Many thanks to all my friends at AIMS, especially Sinenkhosi Mamba, who made me feel at home. Thank you to everyone who has been a constant support on my journey. May this work also serve as a testament to the effort you have put into helping me and the kindness you have shown me.

%-----------------------------------------------------------------------------
% Note the errata page is not for now, it is for use during the examination.
% Not that you're going to have any errata.
%-----------------------------------------------------------------------------

\renewcommand{\bibname}{References}

\bibliographystyle{plain}
\bibliography{main}
\addcontentsline{toc}{chapter}{References}
%-----------------------------------------------------------------------------
\end{document}